\documentclass[twoside]{article}

\usepackage[accepted]{styles/aistats2026}
\usepackage[utf8]{inputenc} 
\usepackage[T1]{fontenc}    
\usepackage{xcolor}
\definecolor{astral}        {RGB}{46,116,181}
\definecolor{cb-blue}       {RGB}{70, 130, 180}
\definecolor{orange}        {RGB}{214,150, 92}
\definecolor{green} {RGB}{136,196,136}
\usepackage{hyperref}
\hypersetup{
colorlinks=true,
allcolors=green
}
\usepackage{fontawesome5}
\usepackage{nicefrac}       
\usepackage{microtype}      
\usepackage{aliascnt}
\usepackage{cleveref}
\usepackage{url}
\usepackage{amsthm, amsfonts}
\usepackage{thmtools}
\usepackage{thm-restate}
\usepackage{natbib}
\usepackage{algorithm}
\usepackage{algpseudocode}
\usepackage{multirow}
\usepackage{booktabs}
\usepackage{subcaption}
\usepackage{wrapfig}
\usepackage{float}
\usepackage{stfloats}

\theoremstyle{plain} 
\newtheorem{theorem}{Theorem}[section]
\newtheorem{proposition}[theorem]{Proposition}

\newtheorem{lemma}[theorem]{Lemma}

\Crefname{assumption_app}{\textbf{A}\hspace{-0.12cm}}{\textbf{H}\hspace{-3pt}}
\crefname{assumption_app}{\textbf{A}}{\textbf{A}}

\theoremstyle{definition} 
\newtheorem{definition}[theorem]{Definition}

\newtheorem{remark}[theorem]{Remark}

\theoremstyle{remark} 

\crefname{theorem}{theorem}{theorems}
\Crefname{theorem}{Theorem}{Theorems}

\crefname{proposition}{proposition}{propositions}
\Crefname{proposition}{Proposition}{Propositions}

\crefname{corollary}{corollary}{corollaries}
\Crefname{corollary}{Corollary}{Corollaries}

\crefname{lemma}{lemma}{lemmas}
\Crefname{lemma}{Lemma}{Lemmas}

\crefname{definition}{definition}{definitions}
\Crefname{definition}{Definition}{Definitions}

\crefname{example}{example}{examples}
\Crefname{example}{Example}{Examples}

\crefname{assumption}{assumption}{assumptions}
\Crefname{assumption}{Assumption}{Assumptions}

\usepackage[acronym]{glossaries}

\newacronym{DM}{{{\textsc{\small DM}}}}{Diffusion Model}
\newacronym{GAN}{{{\textsc{\small GAN}}}}{Generative Adversarial Network}
\newacronym{VAE}{{{\textsc{\small VAE}}}}{Variational Autoencoder}
\newacronym{DiffusionVAE}{{{\textsc{\small DiffusionVAE}}}}{ VAE with a Diffusion Prior}
\newacronym{VI}{{{\textsc{\small VI}}}}{Variational Inference}
\newacronym{VQVAE}{{{\textsc{\small VQ-VAE}}}}{Vector-Quantized Variational Autoencoder}
\newacronym{LDM}{{{\textsc{\small LDM}}}}{Latent Diffusion Model}
\newacronym{EBM}{{{\textsc{\small EBM}}}}{Energy-based Model}
\newacronym{IPLD}{{{\textsc{\small IPLD}}}}{Interacting Particle Latent Diffusion}
\newacronym{DDPM}{{{\textsc{\small DDPM}}}}{Denoising Diffusion Probabilistic Model}
\newacronym{LSGM}{{{\textsc{\small LSGM}}}}{Latent Score-based Generative Models}
\newacronym{SVGD}{{{\textsc{\small SVGD}}}}{Stein Variational Gradient Descent}
\newacronym{MMD}{{{\textsc{\small MMD}}}}{Maximum Mean Discrepancy}
\newacronym{ELBO}{{{\textsc{\small ELBO}}}}{Evidence Lower BOund}
\newacronym{LDDM}{{{\textsc{\small LDDM}}}}{Latent Denoising Diffusion Model}
\newacronym{LVM}{{{\textsc{\small LVM}}}}{Latent Variable Model}
\newacronym{PGD}{{{\textsc{\small PGD}}}}{Particle Gradient Descent}
\newacronym{MPGD}{{{\textsc{\small MPGD}}}}{Momentum Particle Gradient Descent}
\newacronym{EM}{{{\textsc{\small EM}}}}{Expectation-Maximization}
\newacronym{FID}{{{\textsc{\small FID}}}}{Fréchet Inception Distance}
\newacronym{GMM}{{\textsc{\small GMM}}}{Gaussian Mixture Model}
\newacronym{xlsi}{$\text{xLSI}$}{Extended Log-Sobolev Inequality}
\newacronym{xt2i}{$\text{xT}_2\text{I}$}{Extended Talagrand-type Inequality}

\makeatletter
\newcommand{\HEADER}[1]{\ALC@it\underline{\textsc{#1}}\begin{ALC@g}}
\newcommand{\ENDHEADER}{\end{ALC@g}}
\makeatother

\usepackage{multicol}
\usepackage{array}
\usepackage{ragged2e} 
\usepackage{minipage-marginpar} 
\newcolumntype{L}[1]{>{\RaggedRight\arraybackslash}p{#1}}
\newcolumntype{C}[1]{>{\Centering\arraybackslash}p{#1}}
\usepackage{siunitx}

\usepackage[normalem]{ulem}
\usepackage{enumitem}

\usepackage[textwidth=1.8cm, textsize=scriptsize]{todonotes}

\usepackage{amsmath,amsfonts,bm}

\def\dd{{\mathrm{d}}}
\def\sX{{\mathsf{X}}}
\def\sZ{{\mathsf{Z}}}

\def\1{\bm{1}}

\def\vs{{\bm{s}}}

\DeclareMathAlphabet{\mathsfit}{\encodingdefault}{\sfdefault}{m}{sl}
\SetMathAlphabet{\mathsfit}{bold}{\encodingdefault}{\sfdefault}{bx}{n}

\newcommand{\R}{\mathbb{R}}

\newcommand{\KL}{D_{\mathrm{KL}}}

\usepackage{standalone}
\begin{document}

\twocolumn[

\aistatstitle{Training Latent Diffusion Models with Interacting Particle Algorithms}

\aistatsauthor{ Tim Y. J. Wang$^{\dagger}$ \And Juan Kuntz \And O. Deniz Akyildiz$^{\dagger}$ }
\runningauthor{Tim Y. J. Wang, Juan Kuntz, O. Deniz Akyildiz}

\aistatsaddress{ $^{\dagger}$Department of Mathematics, Imperial College London }
]

\makeatletter
\fancyhead[CE]{\small\bfseries\@runningtitle}
\fancyhead[CO]{\small\bfseries\@runningauthor}
\makeatother

\begin{abstract}
  We introduce a novel particle-based algorithm for end-to-end training of latent diffusion models. We reformulate the training task as minimizing a free energy functional and obtain a gradient flow that does so. By approximating the latter with a system of interacting particles, we obtain the algorithm, which we underpin theoretically by providing error guarantees. The novel algorithm compares favorably in experiments with previous particle-based methods and variational inference analogues.
\end{abstract}

\section{INTRODUCTION}
\glspl*{DM} introduced in~\cite{pmlr-v37-sohl-dickstein15}, and further developed in~\cite{hodenoisingdiffusionprobabilistic2020, songscorebasedgenerativemodeling2021}, excel at numerous generative modeling tasks. Examples include image synthesis \citep{dhariwal2021beatsGAN}, protein design \citep{watsonNovoDesignProtein2023}, and language modeling \citep{nieLargeLanguageDiffusion2025}. They work by progressively adding noise to data to transform the data distribution into an easy-to-sample reference distribution, and then learn to revert this noising process.

However, these steps take place in the data's \emph{ambient} space, which is typically high dimensional. For this reason, \gls*{DM} training and inference prove computationally expensive. To alleviate this issue, \cite{rombachhighresolutionimagesynthesis2022,vahdatscorebasedgenerativemodeling2021,wehenkeldiffusionpriorsvariational2021} and others proposed using a \gls*{VAE} \citep{Kingma2013AutoEncodingVB} to map back-and-forth between the high-dimensional ambient space and a low-dimensional \emph{latent} space, and carrying out the noising/de-noising steps in the latter. To date, the world's most popular \glspl*{DM} (e.g., Stable Diffusion~\citep{podell2024sdxl}) fall into this \gls*{LDM} category.

Recently,~\cite{kuntz23a,limMomentumParticleMaximum2024,limParticleSemiimplicitVariational2024} have reported performance gains for parameter estimation in simple latent variable models and generator networks by replacing variational approximations with particle-based ones. Here, we investigate whether this is also the case for \glspl*{LDM} and introduce, to the best of our knowledge, the first particle-based method for \gls*{LDM} training.

\paragraph{Contributions.} 
\begin{enumerate}[label={(C\arabic*)}]
\item We recast the problem of \gls*{LDM} training as the minimization of a free energy functional, characterize the functional's minima (Proposition~\ref{prop:ftildemins}), identify a gradient flow that minimizes it, and establish its exponential convergence under standard assumptions (Theorem~\ref{thm:exp_convg_informal}).
\item Approximating the flow in (C1), we obtain \glspl*{IPLD}---a simple, particle-based, and encoder-free algorithm for \gls*{LDM} training well-adapted to modern compute environments---and we derive non-asymptotic bounds on its error (Theorem~\ref{thrm:error-bound}).
\item Through several practical improvements, we obtain an efficient and scalable version of the algorithm (Section~\ref{sec:pratical-alg}) and demonstrate its effectiveness in numerical experiments (Section~\ref{sec:numerics}).
\item Lastly, by approaching \glspl*{LDM} from the above unexplored angle, we open the door to other novel \gls*{LDM} training methods. In particular, our approach connects latent diffusion models to the rich body of work on gradient flows and interacting particle systems stemming from the optimal transport literature \citep{villani2008optimal,chaintron_2022}--a connection that can spur the design of new algorithms as we demonstrate here.
\end{enumerate}
\paragraph{Paper structure.} The paper is organized as follows. First, we review the necessary background on \glspl*{LDM} and identify relevant loss functions (Section~\ref{sec:dldm}). Next, we obtain \gls*{IPLD}: an algorithm for minimizing this loss (Section~\ref{sec:gradient_flow_ldm}). We do so by identifying a gradient flow that minimizes the loss (Section~\ref{sec:gradient_flow}), approximating it (Section~\ref{sec:practical_disc}), and incorporating a series of practical improvements (Section~\ref{sec:pratical-alg}). We then survey the related literature (Section~\ref{sec:rel_work}) and experimentally compare \gls*{IPLD} with relevant baselines (Section~\ref{sec:numerics}). We conclude with a discussion of our results, \gls*{IPLD}'s limitations, and future research directions (Section~\ref{sec:conclusion}).

\paragraph{Notation.} We use $\mathsf{X} = \mathbb{R}^{d_x}$ and $\mathsf{Z} = \mathbb{R}^{d_z}$ to denote the ambient and latent spaces, $\Theta = \mathbb{R}^{d_\theta}$ and $\Phi = \mathbb{R}^{d_\phi}$ the decoder's and \gls*{DM}'s parameter spaces (c.f.~Section~\ref{sec:dldm}), $\{x^1,\dots,x^M\}$ the training set, $[M]:=\{1,\dots,M\}$ the set of indices, and $\mathcal{P}_2(\R^d)$ the space of probability distributions on $\R^d$ with finite second moment. To denote product measures, we write $q^{1:M} = (q^1, \dots, q^M) \in \mathcal{P}_2(\sZ)^M$ for $M$-tuples of distributions $q^1, \dots, q^M$ over $\sZ$. 

\section{PRELIMINARIES}\label{sec:dldm}

We consider latent-space versions of \glspl*{DDPM}~\citep{hodenoisingdiffusionprobabilistic2020} similar to those in \cite{rombachhighresolutionimagesynthesis2022, vahdatscorebasedgenerativemodeling2021,wehenkeldiffusionpriorsvariational2021}. To be specific, for a fixed data point $x \in \sX$, we consider the following latent variable model:
\begin{equation}\label{eq:LDDM}
    p_{\theta, \phi}(x, z_{0:K}) = p_{\phi}(x|z_{0}) p_{\theta}(z_{0:K})
\end{equation}
where ${p_{\phi}(x|z_{0}) = \mathcal{N}(x; g_\phi(z_0),\sigma^2 I)}$ is an isotropic Gaussian decoder, with $g_\phi: \sZ \to \sX$ 
denoting a neural network parameterized by $\phi$, that  maps from the latent space to the ambient space, and the prior
$$p_{\theta}(z_{0:K}) := p(z_K) \prod_{k=1}^K p_{\theta, k}(z_{k-1}|z_k),$$
is a \gls*{DDPM} parameterized by $\theta$ \citep{hodenoisingdiffusionprobabilistic2020}. The \gls*{DDPM} end point is defined with a standard Gaussian distribution $p_K (z_{K}):=\mathcal{N}(z_K; 0, I)$ at time $K$, and its backward kernels are also Gaussian: $$p_{\theta, k}(z_{k-1}|z_{k}):=\mathcal{N}(z_{k-1}; \mu_{\theta, k}(z_k), \beta_k^2 I),$$ where $(k,z)\mapsto\mu_{\theta, k}(z)$ denotes a neural network parameterized by $\theta$; and $\{\beta_k\}_{k=1}^K$ a fixed noise schedule.

Suppose we are given a dataset $\{x^1,\dots,x^M\}$ where $x^m \sim p_{\text{data}}$ for $m \in [M]$. To fit the generative model in \eqref{eq:LDDM}, we aim to find parameters $(\theta_\star,\phi_\star)$ that maximize the \emph{expected log-likelihood} $\mathbb{E}_{p_{\text{data}}}[\log p_{\theta,\phi}(X)]$. Given that we only have access to the empirical measure $p_{\text{data}}^M = (1/M) \sum_{m=1}^M \delta_{x^m}$, we approximate the expected log-likelihood with the empirical average:
\begin{equation}\label{eq:ell}
    \ell(\theta,\phi):=\frac{1}{M}\sum_{m=1}^M\log p_{\theta,\phi}(x^m),
\end{equation}
where $$p_{\theta,\phi}(x):=\int p_{\theta,\phi}(x,z_{0:K}) \dd z_{0:K}$$ denotes the probability density the model assigns to a given datapoint $x$. This quantity is also called the \emph{marginal likelihood}.
\subsection{Minimizing Free Energy}
The direct computation of $\ell(\theta, \phi)$ (and consequently of its gradients) is intractable, as $p_{\theta,\phi}(x)$ involves marginalising over the latent variables $z_{0:K}$. To circumvent this issue, we instead look at utilising lower bounds on $\ell(\theta, \phi)$. Using the standard lower bound of marginal likelihood, we obtain using Jensen's inequality that $\log p_{\theta,\phi}(x) \geq \mathbb{E}_{q}[\log p_{\theta,\phi}(x, \cdot)/q(\cdot)]$ for any $x$ and distribution $q$ over the latent space. We aim at generalising this bound for $M$ data points where $q$ is the prior distribution defined by the \gls*{DDPM} in the latent space.

To this end, consider the product measure $q^{1:M} \in \mathcal{P}_2(\sZ)^M$ and note that
\begin{align}\label{eq:lower_bound}
\ell(\theta, \phi) &\geq \frac{1}{M}\sum_{m=1}^M\mathbb{E}_{q^m}\left[ \log \frac{p_{\theta, \phi}(x^m, z_{0})}{q^m(z_{0})} \right]
\end{align}
where we have for all $(\theta, \phi)\in \Theta\times \Phi$: 
\begin{equation}\label{eq:joint_z1k}
    p_{\theta, \phi}(\cdot, z_{0}):=\int p_{\theta,\phi}(\cdot,z_{0:K})\dd z_{1:K}.
\end{equation}
The negative of the quantity in the r.h.s. of \eqref{eq:lower_bound} is termed \textit{the free energy}, denoted $F(\theta,\phi,q^{1:M})$
\citep{bishop2006pattern}:
\begin{align}\label{eq:fen78awfnye8awnfeua}
F(\theta,\phi,q^{1:M})&:= \frac{1}{M}\sum_{m=1}^M\mathbb{E}_{q^m}\left[ \log \frac{q^m(z_{0})}{p_{\theta, \phi}(x^m, z_{0})} \right].
\end{align}
Noting that $-\ell(\theta, \phi) \leq F(\theta, \phi, q^{1:M})$ for all $(\theta, \phi, q^{1:M}) \in \Theta \times \Phi \times \mathcal{P}_2(\sZ)^M$, we see that minimizing \eqref{eq:fen78awfnye8awnfeua} above over all parameters $\theta,\phi$ and $M$-tuples $q^{1:M}=(q^1,\dots,q^M)$ is equivalent to maximizing $\ell(\theta,\phi)$ over all $\theta,\phi$; see, e.g.,~\citet{Neal1998}. 

However, the joint density $p_{\theta, \phi}(x, z_{0})$ in \eqref{eq:fen78awfnye8awnfeua} is still computationally prohibitive to evaluate due to marginalizing over the entire diffusion trajectory $z_{1:K}$. We therefore resort to one more upper bound for the negative log-likelihood $-\ell(\theta, \phi)$, leading to the \textit{tilted free energy} (see Appendix~\ref{app:reparam_diffusion} for the full derivation):
\begin{align}
\tilde{F}(\theta, \phi, q^{1:M}) &:= \frac{1}{M}\sum_{m=1}^M\mathbb{E}\left[ \log \frac{q^m(z_{0:K})}{p_{\theta, \phi}(x^m, z_{0:K})} \right],\label{eq:free-energy}
\end{align}
where the expectation is taken with respect to $q^m(z_{0:K})$ with $q^m(z_{0:K}):=q(z_{1:K}|z_0)q^m(z_0)$ for all $m \in [M]$ and $$q(z_{1:K}|z_{0})=\prod_{k=1}^K q(z_k|z_{k-1})$$ is the \emph{forward process}, a product of Gaussian kernels as in \gls*{DDPM}.

Examining the second lower bound in \eqref{eq:free-energy}, we note that $\tilde{F}$ can be decomposed as (cf. Appendix~\ref{app:reparam_diffusion}):
\begin{align*}
    \tilde{F}(\theta, \phi, q^{1:M}) 
    &= F(\theta, \phi, q^{1:M}) + \frac{1}{M}\sum_{m=1}^M \mathbb{E}_{q^m}[\mathcal{R}(\theta, z_0)],
    \end{align*}
    where
\begin{align}
    \quad \mathcal{R}(\theta, z_0) &= \KL\big(q(z_{1:K}|z_{0}) \|p_{\theta}(z_{1:K}|z_{0})\big) \geq 0.
\end{align}
So $\tilde{F}$ amounts to a regularized version of $F$ that penalizes deviations from the forward process.
Therefore, we can alternatively view $\tilde{F}$ as the free energy obtained replacing  $p_{\theta,\phi}(x^m,z_0)$ in~\eqref{eq:fen78awfnye8awnfeua} with the \emph{tilted} version:
\begin{equation}\label{eq:ptildedef}
\tilde{p}_{\theta,\phi}(x^m,z_0):=p_{\theta,\phi}(x^m,z_0)\exp(-\mathcal{R}(\theta,z_0)),
\end{equation}
for $m \in [M]$. For this reason, similar arguments to those behind \cite[Theorem 2]{Neal1998} yield the following result.
\begin{proposition}\label{prop:ftildemins}$(\theta_\star,\phi_\star)$ maximize $\tilde{\ell}(\theta,\phi):=M^{-1}\sum_{m=1}^M\log \tilde{p}_{\theta,\phi}(x^m)$ iff $(\theta_\star,\phi_\star,q^{1:M}_\star)$ minimize $\tilde{F}$ for some $q_\star^{1:M}$, where $\tilde{p}_{\theta,\phi}(x^m) := \int \tilde{p}_{\theta,\phi}(x^m,z_{0:K}) \dd z_{0:K}$ for all $m \in [M]$.
\end{proposition}
In an idealized setting where the backward process parameterized by $\theta$ is expressive enough to match the chosen forward process, the penalty $\mathcal{R}(\theta, z_0)$ vanishes, and our tilted log-likelihood $\tilde{\ell}(\theta,\phi)$ shares the same maxima and maximizers with the true $\ell(\theta,\phi)$ in \eqref{eq:ell}.
Hence, for optimal parameters $\theta_\star$ such that $$\theta_\star \in \Theta_0:=\{\theta\in \Theta: \mathcal{R}(\theta,\cdot)\equiv 0\},$$  Proposition~\ref{prop:ftildemins} implies that the minimizers $(\theta_\star,\phi_\star)$-components of $\tilde{F}(\theta, \phi, q^{1:M})$, our tractable proxy for optimization, also maximize the marginal log-likelihood $\ell$, our objective of interest. 
While this is never perfectly achieved in practice, we believe a close approximation is possible when the forward and reverse processes are chosen with care (generally, the more expressive the latter is, the better).

Previous works such as~\cite{wehenkeldiffusionpriorsvariational2021} restrict $q^m$ to mean-field Gaussian distributions of the sort $\mathcal{N}(z_0;\mu_\psi(x^m),\Sigma_\psi(x^m))$, where the mean and (diagonal) covariance matrix are parameterized by an encoder with parameters $\psi$, and they optimize $\tilde{F}$ over $(\theta,\phi,\psi)$. In the following, we take a different approach and replace these parametric variational approximations with particle-based ones.

\section{INTERACTING PARTICLE LATENT DIFFUSION}\label{sec:gradient_flow_ldm}
To obtain an algorithm for minimizing $\tilde{F}$ in \eqref{eq:free-energy}, we follow similar steps to those taken in \citet{kuntz23a} to obtain particle-based algorithms for minimizing $F$ in \eqref{eq:fen78awfnye8awnfeua} for simpler latent variable models (in particular, ones for which $p_{\theta,\phi}(x,z_0)$ have a standard Gaussian prior and do not have complex dependencies on latent variables $z_0$). 

\subsection{The gradient flow}
\label{sec:gradient_flow}
To derive our algorithm, we search for an analogue of gradient descent (GD) applicable to $\tilde{F}(\theta, \phi, q^{1:M})$. A single update step of GD in the Euclidean space $\R^d$ for minimizing a function $f:\R^d\to\R$ is given by $x_{k+1}=x_k-h\nabla_x f(x_k)$. This update step is exactly the Euler discretization with step size $h>0$ of the continuous-time \emph{gradient flow} $\dot{x}_t=-\nabla_xf(x_t)$.
The analogue of the continuous time gradient flow we now identify resides in the joint space of parameters $\Theta\times \Phi\ni (\theta, \phi)$ and distributions $\mathcal{P}_2(\sZ)^M\ni q^{1:M}$. Under this geometry,
$\nabla \tilde{F}=(\nabla_\theta \tilde{F}, \nabla_\phi \tilde{F},\nabla_{q^1} \tilde{F}, \dots,\nabla_{q^M} \tilde{F})$ (see Appendix~\ref{app:otto_calc} for details):
\begin{align}
    &\nabla_\theta \tilde{F}(\theta, \phi, q^{1:M}) = \frac{1}{M}\sum_{m=1}^M\mathbb{E}_{q^m}\left[\nabla_{\theta} \mathcal{L}_D(\theta, z_{0})\right],
    \label{eq:theta_ode} \\
    &\nabla_\phi \tilde{F}(\theta, \phi, q^{1:M}) 
    =-\frac{1}{M}\sum_{m=1}^M\mathbb{E}_{q^m}\left[\nabla_{\phi} \log p_{\phi}( x^m | z_0)\right],\label{eq:phi_ode} \\\nonumber
    &\nabla_{q^m} \tilde{F}(\theta, \phi, q^{1:M}) \\ 
    &= [\nabla_{z_0} \cdot \left[
        q^m(z_0) \nabla_{z_0}  \mathcal{L}^m(\theta, \phi, z_0) \right]- \Delta_{z_0} q^m(z_0)],\label{eq:fpe_qt} 
\end{align}
where the last equation holds for all $m\in [M]$, and we have defined:
\begin{equation}\label{eq:single_loss_drift}
    \mathcal{L}^m(\theta, \phi, z_0):=\log(p_{\phi}( x^m|z_0))-\mathcal{L}_D(\theta, z_0), 
\end{equation}
with $\mathcal{L}_D(\theta,z_0)$ being the \emph{diffusion loss} for the latent space \gls*{DDPM} prior\footnote{We point out that the loss is dependent on the number of diffusion time steps $K$.}:
\begin{equation}
\mathcal{L}_D(\theta,z_0):=\mathbb{E}_{q(z_{1:K}|z_{0})}\left[ \log \frac{{q(z_{1:K}|z_{0})}}{p_{\theta}(z_{0:K})} \right].
\end{equation}
The gradient flow then reads 
\begin{equation}\label{eq:flow}
(\dot{\theta}_t,\dot{\phi}_t,\dots \dot{q}^{1:M}_t)=-\nabla \tilde{F}(\theta_t,\phi_t,q^{1:M}_t).
\end{equation}
Using the results of \cite{caprioerrorboundsparticle2024}, it is straightforward to obtain sufficient conditions under which the flow converges exponentially fast to $\tilde{F}$'s minimizers. We state a concise version of the result and associated assumptions; we defer the full statement and proof to Appendix~\ref{sec_app:exp_convg_proof}. Let $\tilde{\rho}_{\theta,\phi}(\cdot):=\tilde p_{\theta,\phi}(x, \cdot)$ be the unnormalized tilted posterior, and $\tilde{\pi}_{\theta,\phi}(\cdot) = \tilde\rho_{\theta,\phi}(\cdot)/A_{\theta, \phi}$ be its normalized version with $A_{\theta, \phi}$ being the normalizing constant.
\begin{restatable}[Model regularity]{assumption_app}{modelregularityrestatable}\label{assump:model_reg}
    We assume that
    \begin{enumerate}
        \item For all $z \in \sZ$, $(\theta,\phi) \mapsto \tilde\pi_{\theta,\phi}(z)$ and $(\theta, \phi) \mapsto A_{\theta, \phi}$ are differentiable;
        \item for all $(\theta,\phi) \in \Theta \times \Phi$, $\tilde\pi_{\theta,\phi}$ is twice continuously differentiable;
        \item $\tilde{\rho}_{\theta,\phi}(z)>0$ for all $z\in \sZ$ and $(\theta,\phi) \in \Theta \times \Phi$.
    \end{enumerate}
\end{restatable}
\begin{restatable}[Regularity of solutions]{assumption_app}{solutionregularityrestatable}\label{assump:solution_reg}
    For any initial conditions $(\theta, \phi, q^{1:M})\in \mathcal{M}^{1:M}$,  the gradient flow has a classical solution $(\theta_t, \phi_t,  q_t^{1:M})_{t \geq 0}$ with $(\theta_0, \phi_0, q_0^{1:M}) = (\theta, \phi, q^{1:M})$.  
Furthermore, for all $m\in [M]$ and $t \geq 0$, $q_t^m$ has a Lebesgue density in $\mathcal{C}^{1,2}([0, \infty) \times \sZ, \mathbb{R}^+)$ and $(\theta_t, \phi_t) \in \mathcal{C}^1([0, \infty), \Theta \times \Phi)$.
\end{restatable}
\begin{restatable}[Strong log-concavity]{assumption_app}{stronglogconcavityrestatable}\label{assump:strong_log_concave}
    For all $x\in \sX$, the tilted joint density $\tilde{p}_{\theta, \phi}(x,z)$ is $\lambda$-strongly log-concave in $(\theta, \phi, z)$ for some $\lambda>0$.
\end{restatable}
\begin{theorem}\label{thm:exp_convg_informal}
Suppose Assumptions \Cref{assump:model_reg}-\Cref{assump:strong_log_concave} hold, then $\tilde{\ell}$ has a unique maximizer $(\theta_\star, \phi_\star)$ and the flow converges exponentially fast to it: for $\lambda$ independent of $M$ and all $t\geq 0$, 
\begin{align*}
        \vert\vert(\theta_t,\phi_t)-(\theta_\star,\phi_\star) \vert\vert
        \leq
        \sqrt{\frac{2[\tilde{F}(\theta_0, \phi_0, q_0^{1:M})-\tilde{F}_\star]}{\lambda}}e^{-\lambda t},
\end{align*}
where we denote $\tilde{F}_\star:=\inf_{(\theta,\phi,q^{1:M})}\tilde{F}(\theta,\phi,q^{1:M})$ and $\vert\vert\cdot\vert\vert$ denotes the Euclidean norm.
\end{theorem}
\begin{proof}[Proof (Sketch)]
Given the fact that $(\theta,\phi,z_0)\mapsto \log(\tilde{p}_{\theta,\phi}(x^m,z_0))$ is $\lambda$-strongly concave for each $m$ in $[M]$, it follows that
\begin{align}
\tilde{\ell}(\theta,\phi,z_0^{1:M})
&:=\frac{1}{M}\sum_{m=1}^M\log\tilde{p}_{\theta,\phi}(x^m,z_0^m),\label{eq:tilde-ell}
\end{align}
is also $\lambda$-strongly concave. The result is then an application of the extended log-Sobolev inequality under strong log-concavity (cf. Definition \ref{def:xLSI} and \citet{caprioerrorboundsparticle2024}), whence the exponential convergence follows from Gr{\"o}nwall's inequality.
\end{proof}

\subsection{Approximating the flow and a simple algorithm}
\label{sec:practical_disc}
In almost all cases, (\ref{eq:theta_ode}--\ref{eq:flow}) defines an intractable set of equations and we must approximate it. To do so, we exploit the fact~\cite[Section~2]{kuntz23a} that they form the Fokker-Planck equation satisfied by the law of the following McKean-Vlasov SDE \citep{chaintron_2022}:
\begin{align}
\dd {\theta}_t &= -\frac{1}{M}\sum_{m=1}^M\mathbb{E}_{q^m_t}\left[\nabla_{\theta} \mathcal{L}_D(\theta_t, Z_{0,t}^m)\right] \dd t, \label{eq:gf_theta_ode} \\
    \dd {\phi}_t &= \frac{1}{M}\sum_{m=1}^M\mathbb{E}_{q^m_t}\left[\nabla_{\phi} \log p_{\phi_t}( x^m | Z_{0,t}^m)\right] \dd t, \\
    \dd Z_{0,t}^m &= \nabla_{z_{0}} \mathcal{L}^m(\theta_t, \phi_t, {Z}_{0,t}^m)dt+\sqrt{{2}} \dd W_{t}^m, \label{eq:gf_z_sde}
\end{align}
where we define the last equation for all $m\in [M]$, $q_t^m$ denotes the law of $Z_{0,t}^m$, and $(W_t^{1:M})_{t\geq 0}$ denotes a $d_z\times M$-dimensional Brownian motion. 

However, the laws $(q_t^{1:M})_{t\geq0}$ are unknown and the continuous time axis intractable, so we must approximate the former and discretize the latter. In particular, for each $m$ we approximate the laws using an empirical distribution:
$$
q_t^m(\mathrm{d}z_0) \approx \frac{1}{N}\sum_{n=1}^N\delta_{Z_{0,t}^{m,n}}(\mathrm{d}z_0),
$$ 
which is formed with $N$ weakly-dependent particles ${Z_{0,t}^{m,1},\dots,Z_{0,t}^{m,N}}$ all (approximately) distributed according to $q_t^{m}$; and we discretize the time axis using the Euler-Maruyama scheme:
\begin{align}
    {\theta}_{t+1} &= \theta_t - (h/MN) \sum_{m=1}^M \sum_{n=1}^N\nabla_{\theta} \mathcal{L}_D(\theta_t, Z_{0,t}^{m,n}),\label{eq:disc_theta} 
    \\
    {\phi}_{t+1} &= \phi_t + (h/MN)  \sum_{m=1}^M\sum_{n=1}^N \nabla_{\phi} \log p_{\phi_t}( x^m | Z_{0,t}^{m,n}),\label{eq:disc_phi} 
    \\
    Z_{0,t+1}^{m,n} &= Z_{0,t}^{m,n} + h\nabla_{z_{0}} [\mathcal{L}^m(\theta_t, \phi_t, {Z}_{0,t}^{m,n})]\nonumber \\
    &+\sqrt{{2h}} W_{t}^{m,n}, \quad \forall (m,n) \in [M] \times [N],\label{eq:disc_par}
\end{align}
 where $h>0$ denotes the discretization step size, $W_{1:T}^{1:M,1:N}$ a $T\times M\times N$-dimensional standard Gaussian random variable, and $T$ denotes the total number of steps. Under the conditions in Theorem~\ref{thm:exp_convg_informal}'s premise, and a further Lipschitz gradients assumption, it is straightforward to obtain error bounds for (\ref{eq:disc_theta}--\ref{eq:disc_par}) using the results in~\cite{caprioerrorboundsparticle2024}.
\begin{restatable}[Lipschitz gradient]{assumption_app}{lipschitzgradientrestatable}\label{assump:lipschitz_grad}
The log-likelihood $\tilde\ell^m(\theta, \phi,z_0):=\log \tilde{p}_{\theta,\phi}(x^m,z_0)$ is differentiable and its gradient 
is $L$-Lipschitz for some $L > 0$, \emph{i.e.} for all $(\theta, \phi, z_0), (\theta',\phi', z_0') \in \Theta \times \Phi \times \sZ$,
\begin{align*}
    &\|\nabla \tilde\ell^m(\theta, \phi, z_0) - \nabla \tilde\ell^m(\theta',\phi', z_0')\| \\
    &\qquad \quad \, \leq L\|(\theta, \phi, z_0) - (\theta', \phi', z_0')\|.
\end{align*}
\end{restatable}
\begin{theorem}\label{thrm:error-bound}
Suppose that the premise of Theorem~\ref{thm:exp_convg_informal} and \Cref{assump:lipschitz_grad} hold, and that $\mathcal{R}$ has Lipschitz gradients. For all sufficiently small $h>0$, there exists a constant $C_{h,N}$ of order $\mathcal{O}(h^{1/2}+N^{-1/2})$ independent of $T$, a $\rho \in (0, 1)$, and a $C>0$ independent of $(h,N,T)$, such that
$$\mathbb{E}\left[\vert\vert(\theta_T,\phi_T)-(\theta_\star,\phi_\star) \vert\vert^2\right]^{1/2} \leq C_{h,N}+C\rho^T\quad\forall T\in \mathbb{N},$$
where $(\theta_\star,\phi_\star)$ denote $\tilde{\ell}$'s unique maximizer.
\end{theorem}
\begin{proof}[Proof (Sketch)]
Given the extra Lipschitz gradients assumption on $\log p_{\theta, \phi}(x,z_0)$, the result follows from bounding the spatial and temporal discretization errors separately, which are then combined with the exponential convergence result in \Cref{thm:exp_convg_informal}. We defer the full proof to \Cref{sec_app:discretize_proof}.
\end{proof}
In particular, \Cref{thrm:error-bound} shows that the error can be made arbitrarily small by picking $N,T$ sufficiently large and $h$ sufficiently small.
\subsection{A practical algorithm: IPLD}\label{sec:pratical-alg}
While amenable to theoretical analysis, the algorithm defined by~(\ref{eq:disc_theta}--\ref{eq:disc_par}) performs poorly in practice when applied to models of the sort we are interested in for several reasons. We deal with these one at a time and obtain a practical discretization,  Interacting Particle Latent Diffusion (\textsc{\small IPLD}), which is well-suited to modern computing environments.

\paragraph{Subsampling.} The updates  in (\ref{eq:disc_theta}--\ref{eq:disc_par}) incur a $\mathcal{O}(NMK)$ computational~\footnote{We use the convention in optimization~\citep{Nesterov2004} and only include the dominant gradient evaluation cost.}, which proves prohibitively expensive for all but the smallest of training sets. We overcome this issue by subsampling similarly as in \citet{sgld2011,hodenoisingdiffusionprobabilistic2020}. First, we rewrite~(\ref{eq:disc_theta}--\ref{eq:disc_par}) more concisely as
\begin{align*}
    (\theta_{t+1},\phi_{t+1})&=(\theta_t,\phi_t)-h\nabla_{(\theta,\phi)} \mathcal{L}_t,\\
    Z_{0,t+1}^{1:M,1:N}&=Z_{0,t}^{1:M,1:N}- (MNh)\nabla_{z_0^{1:M,1:N}} \mathcal{L}_t\\
    &+\sqrt{{2h}}W_t^{1:M,1:N},
\end{align*}
where $\mathcal{L}_t:=(MN)^{-1}\sum_{m=1}^M\sum_{n=1}^N\left[\mathcal{L}^m(\theta, \phi, z_{0,t}^{m,n})\right]$ with $\mathcal{L}^m$ defined in \eqref{eq:single_loss_drift}. Then, we replace the loss $\mathcal{L}_t$ with the following unbiased estimate:
\begin{align}
    &\hat{\mathcal{L}}_t := \hat{\mathcal{L}}(\theta_t,\phi_t,z^{1:M,1:N}_{0,t})\label{eq:batch_loss}\\
    &:= \frac{1}{N |\mathcal{B}|} \sum_{(m,n)\in \mathcal{B} \times [N]} \left[
    \hat{\mathcal{L}}_D(\theta_t, z_{0,t}^{m,n}) - \log p_\phi ( x^m | z_{0,t}^{m,n}) \right] \notag
\end{align}
where $\mathcal{B}$ denotes a subset of $[M]$ of size $|\mathcal{B}|$ drawn uniformly at random (and independently of all other random variables), and $\hat{\mathcal{L}}_{\mathcal{D}}$ denotes the $\mathcal{O}(1)$-cost unbiased estimate of the diffusion loss $\mathcal{L}_{\mathcal{D}}$ specified in Appendix~\ref{app:reparam_diffusion}; so bringing down the cost to $\mathcal{O}(N|\mathcal{B}|)$. We add noise to the particles' updates scaled by $\sqrt{|\mathcal{B}|/M}$ to match the variance of the noise in the full-batch updates in~\eqref{eq:disc_par}. See \Cref{algo:train_basic} below for pseudocode.
\begin{algorithm}[H]
    \caption{Interacting Particle Latent Diffusion (\gls*{IPLD})}\label{algo:train_basic}
    \begin{algorithmic}[1]
        \linespread{1.2}\selectfont
        \State \textbf{Inputs:} Dataset $\{ x^m\}_{m\in [M]}$, 
        stepsize $h$,
        \\ particles $z_{0,0}^{1:M,1:N}$, parameters $\phi,\theta$
        \While{not converged}
            \State Sample a mini-batch of indices $\mathcal{B}\subset [M]$
            \State Compute $\hat{\mathcal{L}}_t = \hat{\mathcal{L}}
            (\theta_t,\phi_t,z^{1:M,1:N}_{0,t})$ in \eqref{eq:batch_loss}
        \For{$(m,n)\in [M] \times [N]$}
        \If{$m\in \mathcal{B}$}
        \State ${z_{0,t}^{m,n} \gets z_{0,t}^{m,n} -(MNh)\nabla_{ z_{0}^{m,n}}\hat{\mathcal{L}}_t}$
        \EndIf
        \State $z_{0,t+1}^{m,n} \gets z_{0,t}^{m,n} +\sqrt{\frac{2h|\mathcal{B}|}{M}}W_{t}^{m,n}$ 
    \EndFor
    \State $\theta_{t+1} \gets \theta_t - h \nabla_{\theta} \hat{\mathcal{L}}_t$
    \State $\phi_{t+1} \gets \phi_t - h \nabla_{\phi} \hat{\mathcal{L}}_t$
\EndWhile
\State \textbf{Outputs:} $(\theta_{t}, \phi_t, z_{0,t}^{1:M, 1:N})$
\end{algorithmic}
\end{algorithm}
\paragraph{Distributed Training.} (\ref{eq:disc_theta}--\ref{eq:disc_par}) require $\mathcal{O}(MN)$ memory. However, \gls*{IPLD} is well-suited for distributed training: $Z^{m,n}_{0,t}$'s update for a given pair $(m,n)$ is independent of that for all other pairs, and requires only access to the $m$th datapoint. Thus, in distributed setups, we allocate each accelerator a disjoint subset of the training set and it handles the corresponding updates for all $N$ particles, reducing per-device memory costs and communication costs. This contrasts with autoencoder-based \glspl*{LDM} that necessitate synchronizing gradients for encoders (typically, deep networks) when trained end-to-end. 

\paragraph{Reweighting and annealing.} We re-weight the diffusion loss $\mathcal{L}_{\mathcal{D}}$ similarly as in \citet{hodenoisingdiffusionprobabilistic2020}, as this re-weighting is known to improve sample quality~\citep{songMaximumLikelihoodTraining2021,kingmaunderstandingdiffusionobjectives2023}; see Appendix~\ref{app:reparam_diffusion} for details. Furthermore, we anneal the KL term in the free energy, replacing the summand in~\eqref{eq:free-energy} with: $-\mathbb{E}_{q^m}\left[\log {p_{\phi}( x^m|z_0)}\right] + \gamma_t \KL(q^m(z_0)||p_\theta(z_0))$, where $\gamma_t:[0,\infty)\to(0,\infty)$ denotes the (non-decreasing) annealing schedule. This is a practice commonly used when training deep \glspl*{LVM} to encourage accurate reconstruction during the early stages of training~\citep{sonderby2016ladderVAE,fu2019cyclicalannealingschedulesimple,vahdatNVAEDeepHierarchical2020,vahdatscorebasedgenerativemodeling2021}. Adjusting $\tilde{F}$ correspondingly results in $\gamma_t$ premultiplying $\mathcal{L}_{\mathcal{D}}$ in~\eqref{eq:fpe_qt}; which in turn corresponds to adjusting the noise levels in~(\ref{eq:gf_z_sde},\ref{eq:disc_par}) by multiplying a factor of $\sqrt{\gamma_t}$ (cf. Appendix~\ref{app:kl_anneal_deriv} for details).

\paragraph{Preconditioning and momentum.} When training simpler latent variable models with algorithms similar to ours, \cite{kuntz23a,limMomentumParticleMaximum2024} observed that preconditioning the models' parameters similarly as in RMSProp~\citep{tieleman2012lecture} mitigated ill-conditioning and stabilized the training. \cite{limMomentumParticleMaximum2024} further noted gains in both training speeds and test-time performance by incorporating momentum into the parameter and particle updates. Here we precondition and incorporate momentum for both the parameters and particles. We use optimizers like AdamW~\citep{loshchilov2019decoupledweightdecayregularization} for the former and adaptive Langevin algorithms \citep{Li2016, kim2020stochasticgradientlangevindynamics} for the latter. 
Lastly, we choose different step sizes $h_\theta$, $h_\phi$, and $h_z$ respectively for $\theta$, $\phi$, and the particles to account for the multi-scale nature of the interacting particle system \citep{akyildiz2024multiscaleperspectivemaximummarginal, Pavliotis2008}. We provide the pseudocode for the full training algorithm in Appendix~\ref{app:precond_training_alg}.

\section{RELATED WORK}\label{sec:rel_work}
\paragraph{Interacting Particle Algorithms.}~Interacting particle systems have long been the foundation of much  statistical and optimization methodology, e.g.~\citep{DelMoral2004, kennedy1995swarm}. Their use in algorithms that fit latent variable models by jointly updating model's parameters and a latent-space particle cloud to maximize the likelihood is more recent:~\citet{kuntz23a} proposed several such algorithms and~\citet{limMomentumParticleMaximum2024} improved their performance by incorporating momentum into the algorithms' updates (see also \cite{encinar2024proximalinteractingparticlelangevin,sharrock24a,oliva2024kineticinteractingparticlelangevin,marion2025implicitdiffusionefficientoptimization,akyildiz2025interacting,marks2025learninglatentenergybasedmodels}). 
The theoretical guarantees of our algorithm are derived from the results in \citet{caprioerrorboundsparticle2024}. Similar error bounds have been established in \citet{akyildiz2025interacting} for an alternative approximation to the gradient flow featuring noise in the parameter updates. Also related to our method are numerous works approximating gradient flows with particles to obtain alternatives to conventional \gls*{VI}; e.g.,~\citet{liuSteinVariationalGradient2016,lambertVariationalInferenceWasserstein2022,duncan23svgd, diao2023forwardbackwardgaussianvariationalinference, limParticleSemiimplicitVariational2024}. Lastly, there is a growing body of work exploring training methodologies for generative models in the ambient space using particle-based methods~\citep{arbel2019maximummeandiscrepancygradient,yiMonoFlowRethinkingDivergence2023,franceschiUnifyingGANsScorebased2023,galashovDeepMMDGradient2024,zhouInductiveMomentMatching2025}.

\paragraph{Latent Diffusion Models.}
There have been various attempts at incorporating \glspl*{DM} into \glspl*{LVM} as the prior $p_{\theta}(z)$. \citet{vahdatscorebasedgenerativemodeling2021} proposed to jointly train a continuous-time score-based \gls*{DM}~\citep{songscorebasedgenerativemodeling2021} in the latent space of a deep hierarchical \gls*{VAE} \citep{vahdatNVAEDeepHierarchical2020}. Similarly, \citet{wehenkeldiffusionpriorsvariational2021} considered the joint training of a discrete-time \gls*{DM} with a conventional \gls*{VAE}. \citet{cohen22vqdiff} instead applied a discrete-time diffusion prior to a \gls*{VQVAE} \citep{vandenoord17vqvae}. The seminal work by \citet{rombachhighresolutionimagesynthesis2022} can also be viewed as learning a diffusion prior post training of the \gls*{VAE}. Additionally, several studies have proposed a trainable forward process for diffusion models \citep{kim2022maximum, bartosh2024neural,bartosh2024flow,nielsen2024diffenc}, which can be viewed as generalizations of latent diffusion models.
More recently, combinations of diffusion-based priors with other types of probabilistic models in hierarchical \glspl*{VAE} have been explored, such as \gls*{EBM} \citep{cuiLearningLatentSpace2024} and Variational Mixture of Posteriors prior \citep{kuzina2024hierarchical}. Additionally, the work by \citet{silvestriTrainingConsistencyModels2025} can also be viewed as an \gls*{VAE} with a consistency model \citep{songConsistencyModels2023} learning the aggregated posterior. Concurrent to our work, \citet{leng2025repaeunlockingvaeendtoend} also considers the joint training of a latent diffusion model with the \gls*{VAE}, but their objective is instead based on the REPresentation Alignment (REPA) loss \citep{yuRepresentationAlignmentGeneration2024}, which require forward passes through pretrained vision foundational models like DINO \citep{caron2021emerging,oquab2024dinov2learningrobustvisual}.

\paragraph{Decoder-only Models.} 
Connected to our work are other methods that, similarly to us, optimize latent variables rather than relying on an encoder network. Several studies \citep{han2017abp, nijkamp2019shortrun,nijkamp2020inf,pangLearningLatentSpace2020,nijkamp2022mcmc,yu2023learning} have considered short-run and persistent Langevin dynamics within an \gls*{EM} framework to train latent variable models featuring a top-down generator network. Also related to our work are approaches like \citet{bojanowski2018optimizing,luiseGeneralizationPropertiesOptimal2020} that optimize latent distributions as an alternative or enhancement to \glspl*{GAN}~\citep{goodfellow2014generative}.

\begin{figure}[htbp]
\centering
\includegraphics[width=0.99\linewidth]{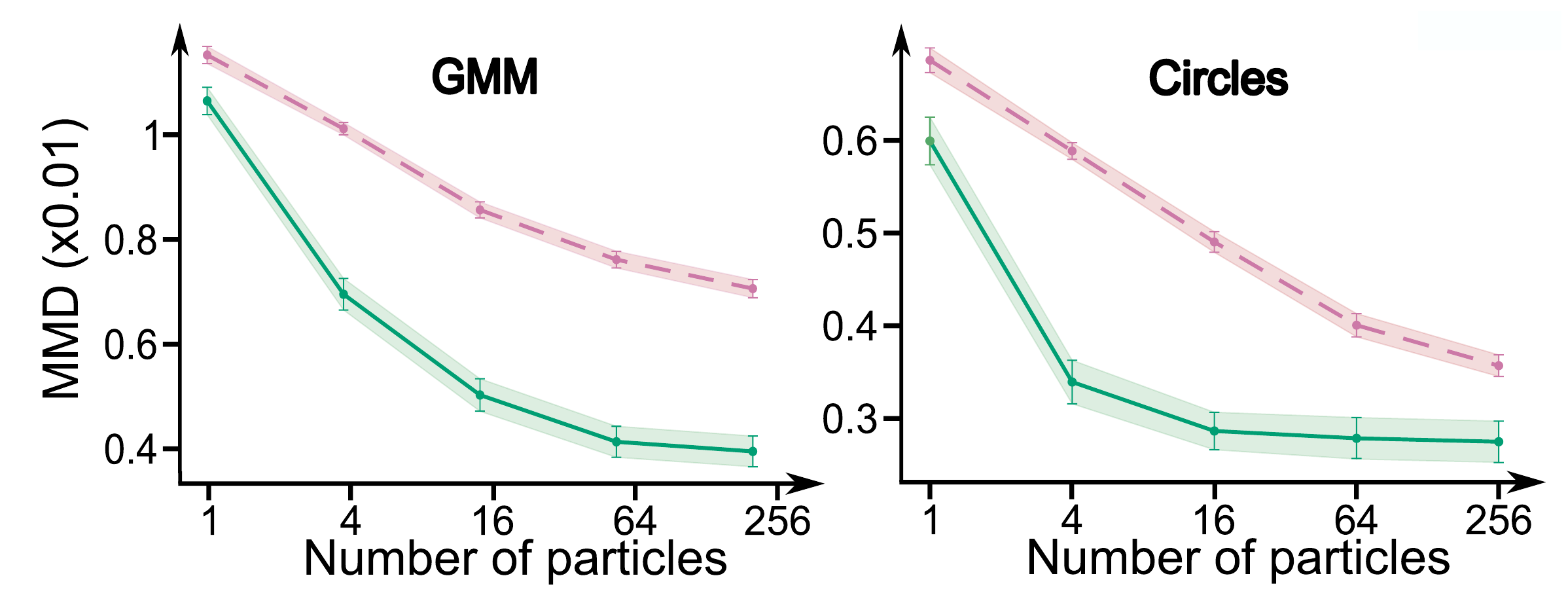}
\caption{Estimated MMD between the ground truth and the distribution learned with \gls*{IPLD} (solid line) and \textsc{\small DiffusionVAE} (dashed line) for the GMM (left) and concentric circles (right) datasets. Shaded regions indicate $\pm1$ standard error.}\label{fig:ldm_gmm_comparison}
\end{figure}
\begin{figure}[htbp]
     \centering
     \includegraphics[width=0.99\linewidth]{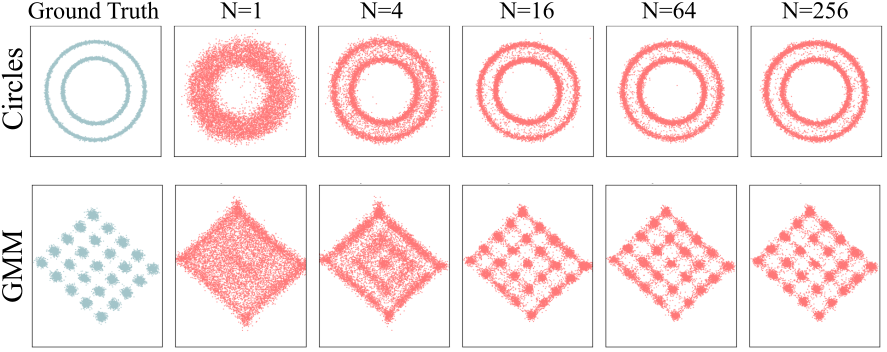}
\caption{Samples generated by \gls*{IPLD} trained with varying numbers of particles for 1,000 steps (showing the first two out of $d_x=64$ dimensions).}\label{fig:toy_examples}
\end{figure}

\begin{table*}[htp]
    \centering
\begin{subtable}[t]{0.45\textwidth}
\centering
\vspace{0pt}
\resizebox{\textwidth}{!}{
\begin{tabular}{l c}
        \toprule
        \textbf{Models} & \textbf{FID($\downarrow$)} \\
        \midrule
        \multicolumn{2}{l}{\textbf{Decoder-only LVMs}} \\
        PGD \citep{kuntz23a} & 101.4 \\
        MPGD \citep{limMomentumParticleMaximum2024} & 91.7 \\
        DAMC \citep{yu2023learning} & 57.72 \\
        LP-EBM \citep{pangLearningLatentSpace2020} & 70.15 \\
        EBM-SR \citep{nijkamp2019shortrun} & 44.50 \\
        \midrule
        \multicolumn{2}{l}{\textbf{VAE}} \\
        DiffusionVAE \citep{wehenkeldiffusionpriorsvariational2021} & 153.1 \\
        {DiffusionVAE}* & 62.07 \\
        \midrule
        \multicolumn{2}{l}{\textbf{Ours}} \\
        IPLD (1 particle)            & 51.60   \\
        IPLD (5 particles)            & 48.30 \\
        IPLD (10 particles)           &  46.95 \\
        \bottomrule
\end{tabular}
}
\caption{\textbf{FID($\downarrow$) of models trained on CIFAR-10}. We report both the original results from \citet{wehenkeldiffusionpriorsvariational2021} and our re-implementation (denoted by "*") of the \gls*{DiffusionVAE} using the same architecture (see \Cref{app:implementaion}).}
\end{subtable}
\begin{minipage}[t]{0.45\textwidth}
    \vspace{0pt}
    \begin{subtable}[t]{\textwidth}
    \centering
    \resizebox{\textwidth}{!}{
    \begin{tabular}{l c}
        \toprule
        \textbf{Models} & \textbf{FID($\downarrow$)}  \\
        \midrule
        DAMC \citep{yu2023learning} & 30.83 \\
        LP-EBM \citep{pangLearningLatentSpace2020} & 37.87 \\
        EBM-SR \citep{nijkamp2019shortrun} & 23.03 \\
        {DiffusionVAE} \citep{wehenkeldiffusionpriorsvariational2021} & 67.95  \\
        {DiffusionVAE}* & 29.24  \\
        \midrule
         \multicolumn{2}{l}{\textbf{Ours}} \\
        IPLD (1 particle)   & 22.86 \\
        IPLD (5 particles)    & 21.55\\
        IPLD (10 particles)  & 21.43 \\
        \bottomrule
    \end{tabular}
}
\caption{\textbf{FID($\downarrow$) of models trained on CelebA64}.}
\end{subtable}
\medskip
\begin{subtable}[t]{\textwidth}
    \centering
    \resizebox{\textwidth}{!}{
    \begin{tabular}{l c}
        \toprule
        \textbf{Models} & \textbf{FID($\downarrow$)} \\
        \midrule
        DAMC \citep{yu2023learning} & 18.76 \\
        LP-EBM \citep{pangLearningLatentSpace2020} & 29.44 \\
        {DiffusionVAE}* \citep{wehenkeldiffusionpriorsvariational2021} & 20.89 \\
        \midrule
         \multicolumn{2}{l}{\textbf{Ours}} \\
        IPLD (1 particle)   & 17.55 \\
        IPLD (5 particles)    & 14.02\\
        IPLD (10 particles)  &  13.51 \\
        \bottomrule
    \end{tabular}
    }
    \caption{\textbf{FID($\downarrow$) of models trained on SVHN}.}
\end{subtable}
\end{minipage}
\caption{FID scores for CIFAR-10, SVHN, and CelebA64 estimated using 50,000 samples.}
\label{tab:fid_scores}
\end{table*}
\section{EXPERIMENTS}
\label{sec:numerics}
We evaluate \gls*{IPLD}'s performance on two
synthetic datasets (Section~\ref{sec:synth}) and three image datasets (Section~\ref{sec:image})~\footnote{Code available at \href{https://github.com/akyildiz-group/IPLD-release}{\faGithub~ IPLD-release}}. In the synthetic case, we benchmark against the closest \gls*{VI}
analogue to \gls*{IPLD} we have been able to locate in the literature: \gls*{DiffusionVAE} \citep{wehenkeldiffusionpriorsvariational2021}; see the end of Section~\ref{sec:dldm}  for more details. For the image datasets, we additionally compare with several other decoder-only \glspl*{LVM} (cf.~Section~\ref{sec:rel_work}).

\subsection{Synthetic Datasets}\label{sec:synth}
We first validate the effectiveness of our method on two toy datasets. We generate these by first sampling from a distribution on a $2$-dimensional latent space, and then mapping the samples into a $64$-dimensional ambient space using a matrix $A \in \R^{2\times 64}$ with orthogonal rows. We consider 1)  a \gls*{GMM} with $25$ components as in \citet{boys2024tweedie,cardoso2024monte}, and 2) a distribution concentrated on concentric circles similar to that\footnote{{\scriptsize\url{https://scikit-learn.org/stable/modules/generated/sklearn.datasets.make_circles.html}}} in \texttt{scikit-learn}.
We train both \gls*{DiffusionVAE} and \gls*{IPLD} for 1,000 steps and varying numbers of particles $N=1,4,16,64,256$ (in the case of \gls*{DiffusionVAE}, $N$ refers to the number of samples drawn from the encoder at each step). For both datasets, increasing $N$ increases the quality of the samples generated by \gls*{IPLD} (\Cref{fig:toy_examples}). To quantitatively compare \gls*{IPLD} and \gls*{DiffusionVAE}, we estimate the \gls*{MMD} \citep{jmlr:v13:gretton12a} between the distribution learned by each and the ground truth. The \gls*{MMD} values are averaged over $50$ training runs and the standard error is reported. For all particle numbers, \gls*{IPLD} outperformed \gls*{DiffusionVAE} (\Cref{fig:ldm_gmm_comparison}). See Appendix~\ref{app:synth} for more details.
\begin{figure}[hbtp]
      \centering
      \includegraphics[width=\linewidth]{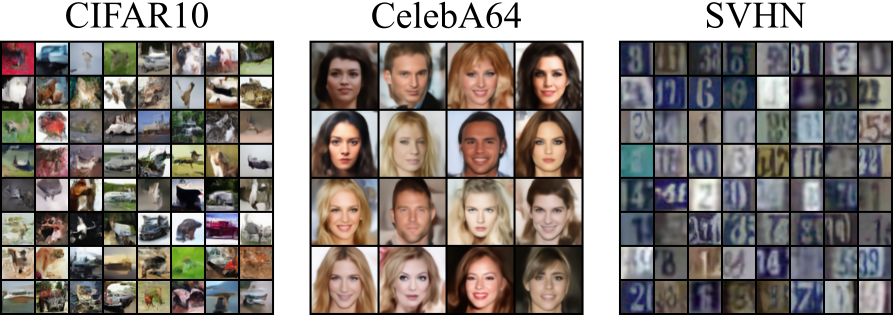}
      \caption{Samples generated with \gls*{IPLD} trained on CIFAR-10, CelebA64, and SVHN. The CelebA64 samples have been curated for better visualization. See Appendix~\ref{app:additional_samples} for additional samples.}
      \label{fig:generated_samples}
 \end{figure}
\subsection{Image Modeling}\label{sec:image} Next,  we test our model on three image datasets: CIFAR-10 \citep{krizhevsky2009learning}, SVHN \citep{svhn2011}, and CelebA64 \citep{liu2015faceattributes}. As before, we benchmark against \gls*{DiffusionVAE} using the same decoder, diffusion model, and $4\times 8\times 8$-dimensional latent space as for \gls*{IPLD}. We train both models for the same number of epochs ($400$ for CIFAR10 and SVHN, $200$ for CelebA64); see Appendix~\ref{app:image_exp_details} for details. For \gls*{IPLD}, we vary the particle number $N=1,5,10$ (we use a single particle for \gls*{DiffusionVAE} because back-propagating the corresponding gradients through the model's encoder with more particles proved too memory-consuming for our hardware). We warm-start the multi-particle \gls*{IPLD} runs using a single particle as in \citet{kuntz23a}; see Appendix~\ref{app:warm_start} for details. In both cases, we compute \gls*{FID} scores for the trained model (\Cref{tab:fid_scores}). \gls*{IPLD} outperforms \gls*{DiffusionVAE}, and its performance increases with the number of particles used for training. 
\gls*{IPLD} also proves competitive with previous decoder-only models discussed in Section~\ref{sec:rel_work} (see also \Cref{tab:fid_scores}).

\section{DISCUSSION}\label{sec:conclusion}
Like many before us, we recast fitting an \gls*{LDM} via maximum likelihood as minimizing a free energy functional (Section~\ref{sec:dldm}). Unlike others, we then identify a gradient flow that minimizes the functional and approximate it using systems of interacting particles (Sections~\ref{sec:gradient_flow}, \ref{sec:practical_disc}), and we theoretically characterize both the flow and the approximations (Theorems~\ref{thm:exp_convg_informal},~\ref{thrm:error-bound}) under standard assumptions. Following these steps, we obtain \gls*{IPLD} (Section~\ref{sec:pratical-alg}): a theoretically-principled algorithm for end-to-end \gls*{LDM} training. Because it entails updating a cloud of particles and each particle's update is independent of the others',  \gls*{IPLD} is well-suited for modern distributed compute environments and is easy to scale. 

In numerical experiments involving both synthetic and image data, \gls*{IPLD} compares favorably with relevant benchmarks (Section~\ref{sec:numerics}). However, our results on the image datasets fall short of today's state-of-the-art. We believe this may be due to the relatively small latent spaces, decoder, and DM architectures we use (e.g., compare Appendix~\ref{app:image_exp_details} with \citet[Appendices~G.1,2]{vahdatscorebasedgenerativemodeling2021})
and our limited computational budget (e.g., compare Appendix~\ref{app:computational_considerations} with \citet[Appendix E]{rombachhighresolutionimagesynthesis2022}), rather than a fundamental limitation of the approach.
Our work may be limited by the fact that, to date, \glspl*{LDM} trained in a two-stage manner (e.g.,~\cite{rombachhighresolutionimagesynthesis2022}) have achieved results that those trained end-to-end have not, and our approach is fundamentally an end-to-end one. However, recent works such as~\cite{leng2025repaeunlockingvaeendtoend} have demonstrated promising breakthroughs in end-to-end \gls*{LDM} training, and their innovations are relatively straightforward to incorporate in our algorithm. Further improvements may be possible through more careful subsampling schemes than that in Section~\ref{sec:pratical-alg} and the use of variance reduction techniques \citep{zou2018subsampled}. Indeed, we hope our work paves the way to other more effective particle-based algorithms for \gls*{LDM} training.

Lastly, our theoretical analysis is limited to the simplified algorithm in Section~\ref{sec:practical_disc}. However, we believe that it may be possible to extend the analysis to more practical versions using techniques along the lines of those used to study adaptive optimizers \citep{malladi2024sdesscalingrulesadaptive} and momentum-enriched interacting particle systems \citep{oliva2024kineticinteractingparticlelangevin,limMomentumParticleMaximum2024}.
\section*{Acknowledgements}
We thank the anonymous reviewers for their constructive comments. TW is supported by the Roth
Scholarship from the Department of Mathematics, Imperial College London. We acknowledge computational resources and support provided by the Department of Mathematics and the Imperial College Research Computing Service, DOI: 10.14469/hpc/2232.

\bibliographystyle{apalike}
\bibliography{bibliography}

\appendix
\thispagestyle{empty}
\onecolumn
\aistatstitle{Supplementary Material for Training Latent Diffusion Models with Interacting Particle Algorithms}

\section{Derivations}\label{app:A}
\subsection{Derivation of the Training Objective}\label{app:reparam_diffusion}

We derive the training objective, the tilted free energy $\tilde{F}(\theta, \phi, q^{1:M})$ in \eqref{eq:free-energy}  and include a derivation of the standard reparametrized diffusion objective $\mathcal{L}_D(\theta, z_0)$ analogous to \citet{hodenoisingdiffusionprobabilistic2020, pmlr-v37-sohl-dickstein15}. For convenience, we re-state the objective for a single data point here and omit the superscript $m$:
\begin{equation}\label{eq:app_free_energy}
    \tilde{F}(\theta, \phi, q) := \mathbb{E}_{q(z_{0:K})}\left[ \log \frac{q(z_{0:K})}{p_{\theta, \phi}( x, z_{0:K})} \right].
\end{equation}

\paragraph{Negative Lower Bound.} To see that $-\tilde{F}(\theta, \phi, q)$ is a lower bound of the log-likelihood $\log p_{\theta, \phi}(x)$, we first note that the tilted free energy in \eqref{eq:app_free_energy} can be re-written as:
\begin{align*}
    \tilde{F}(\theta, \phi, q) &= \mathbb{E}_{q(z_0) q(z_{1:K}|z_0)}\left[ \log \frac{q(z_0) q(z_{1:K}|z_0)}{p_{\theta, \phi}(x|z_{0})p_\theta(z_{1:K}|z_0)p_\theta(z_0)} \right] \\
    &= \mathbb{E}_{q(z_0)}\left[ \mathbb{E}_{q(z_{1:K}|z_0)}\left[ \log \frac{q(z_0)}{p_{\theta, \phi}(x, z_0)} + \log \frac{q(z_{1:K}|z_0)}{p_\theta(z_{1:K}|z_0)} \right] \right] \\
    &= \mathbb{E}_{q(z_0)}\left[ \log \frac{q(z_0)}{p_{\theta, \phi}(x, z_0)} \right] + \mathbb{E}_{q(z_0)}[\KL(q(z_{1:K}|z_0) \| p_\theta(z_{1:K}|z_0))],
\end{align*}
where we have assumed the independence structure of the decoder $p_{\phi}( x|z_{0:K})=p_{\phi}( x|z_0)$; the first term is the usual free energy $F(\theta, \phi, q)=\mathbb{E}_{q(z_0)}\left[ \log {q(z_0)}-\log{p_{\theta, \phi}(x, z_0)} \right]$ in \eqref{eq:fen78awfnye8awnfeua} and the second term is non-negative. 
It is thus easy to see that the tilted free energy $\tilde{F}(\theta, \phi, q)$ is obtained by replacing $p_{\theta,\phi}(x^m,z_0)$ in~\eqref{eq:fen78awfnye8awnfeua} with the tilted $\tilde{p}_{\theta,\phi}(x^m,z_0):=p_{\theta,\phi}(x^m,z_0)\exp(-\mathcal{R}(\theta,z_0))$, where $\mathcal{R}(\theta, z_0)= \KL\big(q(z_{1:K}|z_{0}) \|p_{\theta}(z_{1:K}|z_{0})\big)$. 

Using Jensen's inequality, we see the negative of the first term is an upper bound of $-\log p_{\theta, \phi}(x)$:
\begin{align*}
    \mathbb{E}_{q(z_0)}\left[ -\log \frac{p_{\theta, \phi}(x, z_0)}{q(z_0)} \right] \geq -\log \mathbb{E}_{q(z_0)}\left[ \frac{p_{\theta, \phi}(x, z_0)}{q(z_0)} \right] = -\log p_{\theta, \phi}(x),
\end{align*}
which leads to the sequence of inequalities:
\begin{equation}
    \tilde{F}(\theta, \phi, q) \geq F(\theta, \phi, q) \geq -\log p_{\theta, \phi}(x).
\end{equation}

An alternative form of the objective $\tilde{F}(\theta, \phi, q)$ amenable to computation can be obtained by decomposing it as follows:
\begin{align}
    \tilde{F}(\theta, \phi, q) &= \int \underbrace{\int \log \frac{q(z_{1:K}|z_{0})}{{p_{\theta}(z_{0:K})}} q(z_{1:K}|z_{0}) \dd z_{1:K} \ }_{\text{Diffusion Objective}  \ \mathcal{L}_D(\theta, z_0)}q(z_{0})\dd z_{0} - {\int \log \frac{{p_{\phi}( x|z_{0})}}{q(z_{0})} q(z_{0}) \ \dd z_{0}}.\label{eq_app:decomposed_obj}
\end{align}

\subsubsection{Derivation of the Diffusion Objective}
The diffusion objective $\mathcal{L}_D(\theta, z_0)$ can be re-written as:
\begin{align}
    \mathcal{L}_D(\theta, z_0) 
    &=\int\left(-\log p_K(z_K)+\log \frac{q_1(z_1|z_0)}{p_{\theta, 1} (z_0|z_1)}+\sum_{k=2}^K \log \frac{q_k(z_k| z_{k-1})}{p_{\theta, k}(z_{k-1}|z_k)}\right)q(z_{1:K}|z_{0})\dd z_{1:K}.\label{eq:app_R}
\end{align}
We can re-write $q_k(z_k|z_{k-1})$ using Bayes' rule as:
\begin{align*}
q_k(z_{k}|z_{k-1})
&=q_k(z_{k-1}|z_{k},z_0)\frac{q(z_{k}|z_0)}{q(z_{k-1}|z_0)}.
\end{align*}

Combined with the Markov assumption $q(z_{1:K}|z_{0})=q(z_K|z_0)\prod_{k=2}^K q_k(z_{k-1}| z_{k}, z_{0})$,  we can rewrite \eqref{eq:app_R} above as:
\begin{align}
    \mathcal{L}_D(\theta, z_0) 
    &= \mathbb{E}_{q(z_{1:K}|z_0)}\left[ \KL(q(z_K|z_0) \| p(z_K)) - \log p_{\theta, 1} (z_0|z_1) \right]  \\
    &+ \mathbb{E}_{q(z_{1:K}|z_0)}\left[\sum_{k=2}^K \KL ({q_k(z_{k-1}| z_{k}, z_0)}\|{p_{\theta, k}(z_{k-1}|z_k)}) \right]
\end{align}
\paragraph{Gaussian Transition Kernel.}
For the forward process, we take
\begin{equation}
    q_k(z_k| z_{k-1})=\mathcal{N}(z_k; \sqrt{1-\beta_k}z_{k-1}, \beta_k I),
\end{equation}
where $\{\beta_k\}_{k=1}^K$ is a linear noise schedule with $\beta_k = (1-k/K)\beta_0 + (k/K)\beta_K$. Using the notation $\alpha_{k}:=1-\beta_{k}$ and $\bar{\alpha}_{k}:=\prod_{j=1}^{k}\alpha_{j}$, we can derive the $k$-step transition kernel:
\begin{equation}
    q(z_k|z_0) = \mathcal{N}(z_k; \sqrt{\bar{\alpha}_k}z_0, (1-\bar{\alpha}_k)I), \quad \forall k=1,\ldots, K.
\end{equation}
By Bayes' rule, we can compute:
\begin{align}
    q_k(z_{k-1} | z_{k}, z_0) &= \mathcal{N}(z_{k-1}; \tilde{\mu}_k(z_k, z_0), \beta_k I), \\
    \tilde{\mu}_{k}(z_k,z_0) &= \frac{\sqrt{\bar{\alpha}_{k-1}}\beta_{k}}{1-\bar{\alpha}_{k}}{z}_{0}+\frac{\sqrt{\alpha_{k}}(1-\bar{\alpha}_{k-1})}{1-\bar{\alpha}_{k}}{z}_{k}.
\end{align}
For the backward process, we set:
\begin{align}
    p_K(z_K) &= \mathcal{N}(0, I) \\
    p_{\theta, k}(z_{k-1}|z_{k})&=\mathcal{N}(z_{k-1}; \mu_{\theta, k}(z_k), \beta_k I) \quad \forall k=2, \ldots, K \\
    p_{\theta, 1}(z_{0}|z_1) &= \mathcal{N}(z_{k-1}; \mu_{\theta, k}(z_k),I)
\end{align}
Using the formula for the Kullback-Leibler divergence between isotropic Gaussian distributions, we obtain:
\begin{align}
        \KL ({q_k(z_{k-1}| z_{k}, z_0)}\|{p_{\theta, k}(z_{k-1}|z_k)}) &=\frac{1}{2}\log\left(\frac{1-\bar{\alpha}_k}{{\beta}_k}\right)-\frac{d_z}{2}+\frac{d_z{\beta}_k+\|\tilde{\mu}_{k}(z_k,z_0)-\mu_{\theta, k}(z_k)\|^2}{2 (1-\bar{\alpha}_k)} \\
        \KL({q(z_K|z_0)}\|{p_\theta(z_K)})&=-\log(1-\bar{\alpha}_K)+\frac{d_z(1-\bar{\alpha}_K)^2-d_z+\bar{\alpha}_K\|{z_0}\|^2}{2} \label{eq_app:extra_term} \\
         -\log(p_\theta(z_0|z_1))&=\frac{d_z}{2}\log(2\pi (1-\bar{\alpha}_1))+\frac{\|z_0-\mu_{\theta, 1}(z_1)\|^2}{2(1-\bar{\alpha}_1)}\label{eq_app:discrete_gaussian}.
\end{align}
Note that in contrast to the standard pixel-space \gls*{DDPM}, we have the extra term in \eqref{eq_app:extra_term} depending on the latent $z_0$; we also use a Gaussian distribution with identity variance for $p_\theta (z_0|z_1)$ in \eqref{eq_app:discrete_gaussian} instead of the independent discrete decoder from \citet{hodenoisingdiffusionprobabilistic2020}. 

\paragraph{Reparametrization.} We adopt the same $\epsilon$-prediction reparameterization as in \citet{hodenoisingdiffusionprobabilistic2020} to write:
\[
\mu_{\theta, k}({z}_k) = \tilde{\mu}_k \left( {z}_k, \frac{1}{\sqrt{\alpha_k}} \left( {z}_k - \sqrt{1 - \bar{\alpha}_k} \, \epsilon_{\theta, k}({z}_k) \right) \right) 
= \frac{1}{\sqrt{\alpha_k}} \left( {z}_k - \frac{\beta_k}{\sqrt{1 - \bar{\alpha}_k}} \epsilon_{\theta, k}({z}_k) \right),
\]
which gives:
\begin{align}
        \KL ({q_k(z_{k-1}| z_{k}, z_0)}\|{p_{\theta, k}(z_{k-1}|z_k)}) &=\frac{1}{2}\log\left(\frac{1-\bar{\alpha}_k}{{\beta}_k}\right)-\frac{d_z}{2}+\frac{d_z{\beta}_k+\beta_k\|\epsilon-\epsilon_{\theta, k}(z_k)\|^2}{2 \alpha_k (1-\bar{\alpha}_k)},
\end{align}
where $\epsilon \sim \mathcal{N}(0, I)$ is a random vector in $\R^{d_z}$ sampled independently from a standard Gaussian. Additionally, using the forward $k$-step transition kernel, we can write $z_k = z_k(z_0, \epsilon)=\sqrt{\bar{\alpha}_k}z_0+\sqrt{1-\bar{\alpha}_k}\epsilon$. Up to a constant term independent of $\theta$, we have the loss function $\mathcal{L}_D(\theta, z_0; \epsilon)\approx \hat{\mathcal{L}}(\theta,z_0;\epsilon, k)$, where we sample $k \sim \text{Unif}(1, \ldots, K)$ and $\epsilon \sim \mathcal{N}(0, I)$. We thus have the following loss:
\begin{align}\label{app_eq:diffusion_likelihood}
    \hat{\mathcal{L}}_D(\theta,z_0;\epsilon, k):=\left\{
    \begin{array}{ll}\frac{1}{2}\|{z_0-\mu_{\theta,1}(z_1(z_0,\epsilon))\|}^2 + \frac{\bar{\alpha}_K\|{z_0\|}^2}{2}&\text{if }k=1\\
    \frac{\beta_k}{2\alpha_k(1-\bar{\alpha}_k)}\|{\epsilon-\epsilon_{\theta,k}(z_k(z_0,\epsilon))\|}^2 + \frac{\bar{\alpha}_K\|{z_0\|}^2}{2}&\text{if }2\leq k\leq K\\
    \end{array}
    \right.
\end{align}
In practice, we follow the approach in \citet{hodenoisingdiffusionprobabilistic2020} to use a simplified version of the diffusion objective:
\begin{align}\label{app:one_sample_simplified_diffusion}
    \hat{\mathcal{L}}_{\text{simple}}(\theta, z_0; \epsilon)\approx \hat{\mathcal{L}}_{\text{simple}}(\theta,z_0;\epsilon, k):=\left\{
    \begin{array}{ll}\frac{1}{2}\|{z_0-\mu_{\theta,1}(z_1(z_0,\epsilon))\|}^2 + \frac{\bar{\alpha}_K\|{z_0\|}^2}{2}&\text{if }k=1\\
    \|{\epsilon-\epsilon_{\theta,k}(z_k(z_0,\epsilon))\|}^2 + \frac{\bar{\alpha}_K\|{z_0\|}^2}{2}&\text{if }2\leq k\leq K\\
    \end{array}
    \right.
\end{align}
More specifically, we draw a minibatch of indices $\mathcal{B} \subseteq [M]$, standard Gaussian random variables $\epsilon^{m,n} \sim \mathcal{N}(0,I)$, and $t^{m,n} \sim \text{Unif}(1, \ldots, K)$ for $m\in \mathcal{B}, n\in [N]$ to approximate the loss by:
\begin{equation}\label{eq_app:simplified_diffusion}
    \hat{\mathcal{L}}_{\text{simple}}(\theta, z_0^{\mathcal{B},1:N}) = \frac{1}{|\mathcal{B}|N} \sum_{(m,n)\in \mathcal{B}\times [N]} \hat{\mathcal{L}}_{\text{simple}}(\theta, z_0^{m,n}; \epsilon^{m,n}, t^{m,n}).
\end{equation}
To simplify the notation, we will denote the one-sample approximation  $\hat{\mathcal{L}}_{\text{simple}}(\theta, z_0^{m,n}; \epsilon^{m,n}, t^{m,n})$ with $\hat{\mathcal{L}}_{\text{simple}}(\theta, z_0^{m,n})$ as in \Cref{sec:pratical-alg} and we use the same simplified notation in the rest of the appendix.

\subsection{The Euclidean-Wasserstein Geometry on \texorpdfstring{$\R^{d_{\theta}} \times \R^{d_\phi}\times \mathcal{P}_2(\R^{d_z})^{M}$}{Theta x Phi x P2(Rdz)^M}}
\label{app:otto_calc}
To obtain the gradient flow in \Cref{sec:gradient_flow}, we view the product space of parameters and probability distributions $\R^{d_{\theta}} \times \R^{d_\phi}\times \mathcal{P}_2(\R^{d_z})^M$ as a Riemannian manifold. We will equip this product space with suitable tangent spaces and Riemannian metrics, which enable us to define the gradients and perform optimization. We omit the subscript on $z_0$ and denote it by $z$ for simplicity throughout this section.

\paragraph{Tangent spaces.} Using the same set-up as in \citet{kuntz23a}, we concatenate the parameters as $\vartheta:=(\theta,\phi)\in \R^D$, where $D=d_\theta+d_\phi$, and assume the approximate posterior is a distribution with strictly positive density w.r.t. the Lebesgue measure and support $\sZ=\R^{d_z}$. We let the product manifold be $\mathcal{M}^{1:M}:=\R^D \times \mathcal{P}_2(\sZ)^M$. For each point $(\vartheta, q^{1:M}) \in \mathcal{M}^{1:M}$, we can define the tangent space $T\mathcal{M}^{1:M}$ and its dual $T^*\mathcal{M}^{1:M}$ as:
\begin{align*}
    T_{(\vartheta, q^{1:M})}\mathcal{M}^{1:M} &= T_\vartheta\R^D \times \prod_{m\in[M]} T_{q^{m}}\mathcal{P}_2(\mathsf{Z})\\
    T_{(\vartheta, q^{1:M})}^*\mathcal{M}^{1:M} &= T_\vartheta^*\R^D \times \prod_{m\in[M]} T_{q^{m}}^*\mathcal{P}_2(\mathsf{Z})
\end{align*}
where we note that $T_\vartheta \R^D \cong T_\vartheta^* \R^D \cong \R^{D}$. For each $q^m \in \mathcal{P}_2(\sZ)$, we also define the tangent and cotangent space of $\mathcal{P}_2(\sZ)$ at $q^m$ as in \citet{otto31012001}:
\begin{align*}
    T_{q^m}\mathcal{P}_2(\mathsf{Z}) &:= \left\{
    r:\mathsf{Z} \to \R: \int r(z) \dd z=0
    \right\} \\
    T_{q^m}^* \mathcal{P}_2(\mathsf{Z}) &:= \left\{
    f: \sZ \to \R
    \right\} / \R
\end{align*}
where the cotangent space $T_{q^m}^*\mathcal{P}_2(\mathsf{Z})$ is identified with the space of equivalence classes of functions that differ by an additive constant.

Furthermore, we define the \textit{duality pairing}, which is a triplet $(T_{(\vartheta, q^{1:M})}, T_{(\vartheta, q^{1:M})}^*, \langle\cdot,\cdot\rangle)$ with $$\langle\cdot,\cdot\rangle: T_{(\vartheta, q^{1:M})}\mathcal{M}^{1:M} \times T_{(\vartheta, q^{1:M})}^*\mathcal{M}^{1:M} \to \R$$ being a bilinear map that is the sum of the Euclidean inner product on the parameter space and the duality pairing on the Wasserstein-2 space:
\begin{align*}
    \langle (\tau, r^{1:M}), (v, f^{1:M}) \rangle &:= \langle \tau, v \rangle + \sum_{m\in[M]} \langle r^m, f^m \rangle,
\end{align*}
where the Wasserstein-2 duality pairing is $(T_{q^m}\mathcal{P}_2(\sZ),T_{q^m}^*\mathcal{P}_2(\sZ), \langle\cdot,\cdot\rangle)$ with $\langle\cdot,\cdot\rangle: T_{q^m}\mathcal{P}_2(\sZ) \times T_{q^m}^*\mathcal{P}_2(\sZ) \to \R$ given by:
\begin{equation}
    \langle r^m,f^m\rangle:=\int f^m(z)r^m(z) \dd z, \quad \forall r^m \in T_{q^m}\mathcal{P}_2(\sZ), f^m\in T_{q^m}^*\mathcal{P}_2(\sZ).
\end{equation}

\paragraph{The metric.}\label{app:metric} We can thus equip the manifold $\mathcal{M}^{1:M}$ with Riemannian metric $g=(g_{(\vartheta, q^{1:M})})_{(\vartheta,q^{1:M})\in \mathcal{M}^{1:M}}$
defined as:
$$
g_{(\vartheta, q^{1:M})}((\tau, r^{1:M}), (\tau', r'^{1:M})) := \langle (\tau, r^{1:M}), G_{(\vartheta, q^{1:M})}(\tau', r'^{1:M}) \rangle, \quad \forall (\vartheta, q^{1:M}) \in \mathcal{M}^{1:M},
$$
where $G_{(\vartheta, q^{1:M})}: T\mathcal{M}^{1:M} \to T^*\mathcal{M}^{1:M}$ is an invertible, self-adjoint, and positive-definite linear map. We will only consider tensors in a block-diagonal form (cf. \citet[Chapter~2]{lee2018riemannian}), which defines the metric via:
\begin{equation}
    \langle (\tau, r^{1:M}), G_{(\vartheta, q^{1:M})}(\tau', r'^{1:M}) \rangle = \langle \tau, \mathrm{G}_{\vartheta} \tau' \rangle + \frac{1}{M}\sum_{m\in[M]} \langle r^m, \mathbf{G}_{q^m}^W r'^m \rangle  \quad \forall (\tau, r^{1:M}) \in T\mathcal{M}^{1:M}, \ (\vartheta, q^{1:M}) \in \mathcal{M}^{1:M}.
\end{equation}
Here we take $\mathrm{G}_{\vartheta}: T\R^D \to T^* \R^D$ as the identity map (which corresponds to the usual Euclidean metric) and $\mathbf{G}_{q}^W: T\mathcal{P}_2(\mathsf{Z}) \to T^*\mathcal{P}_2(\mathsf{Z})$ to be the tensor for Wasserstein-2 distance on $\mathcal{P}_2(\mathsf{Z})$ defined through its inverse:
\begin{equation}
    (\mathbf{G}_q^W)^{-1}f:=-\nabla_{z} \cdot (q \nabla_{z}f), \quad \forall f\in C^\infty(\mathcal{P}_2(\mathsf{Z})).
\end{equation}
When the context is clear, we will also write $\langle\tau, \tau'\rangle_{\vartheta} =\langle \tau, \mathrm{G}_{\vartheta} \tau' \rangle$ and $\langle r^m, r'^m \rangle_{q^m} = \langle r^m, \mathbf{G}_{q^m}^W r'^m \rangle$.

\paragraph{The gradient.} To perform gradient descent on manifolds, we further need an analogue of the gradient for a smooth function $F: \mathcal{M}^{1:M}\to \R$ as in the Euclidean space. This is a vector field $\nabla F: \mathcal{M}^{1:M}\to T\mathcal{M}^{1:M}$ satisfying:
\begin{equation}
    \begin{aligned}
    g_{(\vartheta, q^{1:M})} (\nabla F(\vartheta, q^{1:M}), (\tau, r^{1:M}))
    &= \lim_{t \to 0} \frac{F(\vartheta + t\tau, q^{1:M} + tr^{1:M}) - F(\vartheta, q^{1:M})}{t} \\
    &\quad \forall (\tau, r^{1:M}) \in T\mathcal{M}^{1:M},\; (\vartheta, q^{1:M}) \in \mathcal{M}^{1:M}.
    \end{aligned}
\end{equation}
By expressing in local coordinates, we can compute the gradient as:
\begin{equation}\label{eq:app_inv_tensor_grad}
    \nabla F(\vartheta, q^{1:M}) = G_{(\vartheta, q^{1:M})}^{-1} \delta F(\vartheta, q^{1:M}), \quad \forall (\vartheta, q^{1:M})\in \mathcal{M}^{1:M},
\end{equation}
where $\delta F: \mathcal{M}^{1:M}\to T^*\mathcal{M}^{1:M}$ is the first variation of $F$ defined as the unique cotangent vector field satisfying:
\begin{equation}\label{eq:app_first_variation}
    \langle (\tau, r^{1:M}), \delta F(\vartheta, q^{1:M}) \rangle = \lim_{t \to 0} \frac{F(\vartheta + t\tau, q^{1:M} + tr^{1:M}) - F(\vartheta, q^{1:M})}{t} 
\quad \forall (\tau, r^{1:M}) \in T\mathcal{M}^{1:M},\; (\vartheta, q^{1:M}) \in \mathcal{M}^{1:M}.
\end{equation}
\paragraph{The distance function.} To define the distance on the manifold $\mathcal{M}^{1:M}$, we recall that the Riemannian distance function is defined as the length of the minimizing geodesics $\gamma: [0,1]\to \mathcal{M}^{1:M}$ between $(\vartheta, q^{1:M})$ and $(\vartheta', q'^{1:M})$:
\begin{align}
    \mathsf{d}_{\mathcal{M}^{1:M}}((\vartheta, q^{1:M}), (\vartheta', q'^{1:M}))^2 := \inf_{\gamma} \int_0^1 g_{(\vartheta_t,q_t^{1:M})}(\dot{\gamma}(t), \dot{\gamma}(t)) \dd t
    &= \|\vartheta-\vartheta'\|_2^2 + \frac{1}{M}\sum_{m=1}^M \mathsf{d}_{W_2}(q^m, (q')^m)^2,
\end{align}
where the second equality follows from~\citet[Section 4.3]{otto31012001} and $W_2$ is the Wasserstein-2 distance on $\mathcal{P}_2(\sZ)$.

\subsection{Derivation of the Gradient Flow}
\label{app:grad_deriv}
Equipped with the tools above, we now derive the gradients of the averaged free energy $\tilde{F}(\theta, \phi, q^{1:M})=M^{-1}\sum_{m=1}^M \tilde{F}^m(\theta, \phi, q^m)$, where $\tilde{F}^m$ is the single-datapoint free energy for $x^m$. We first compute the single-datapoint first variation \eqref{eq:app_first_variation}, and then lift it componentwise to the averaged objective:
\begin{lemma}\label{app:lemma_1st_var}
     Given the free energy of the form in \eqref{eq_app:decomposed_obj}:
\begin{equation}
    \tilde{F}(\theta, \phi, q)  =  \int \mathcal{L}_D(\theta, z) q(z)\dd z - \int \log \left(\frac{p_\phi( x|z)}{q(z)}\right) q(z)\dd z.
\end{equation}
The first variation $\delta \tilde{F}(\theta, \phi, q)=(\delta_q \tilde{F}(\theta, \phi, q), \delta_\theta \tilde{F}(\theta, \phi, q), \delta_\phi \tilde{F}(\theta, \phi, q))$ is given by:
\begin{align}
    \delta_q \tilde{F}(\theta, \phi, q) &= \log \left(\frac{q(z)}{p_\phi( x|z)}\right) + \mathcal{L}_D(\theta, z)\label{eq_app:1st_var_q}   \\
    \delta_\theta \tilde{F}(\theta, \phi, q) &= \int \nabla_\theta \mathcal{L}_D(\theta, z) q(z) \dd z\label{eq_app:1st_var_theta}  \\
    \delta_\phi \tilde{F}(\theta, \phi, q) &= - \int \nabla_\phi \log p_\phi( x|z) q(z) \dd z\label{eq_app:1st_var_phi} 
\end{align}
\end{lemma}
\begin{proof}
    For \cref{eq_app:1st_var_theta,eq_app:1st_var_phi}, it suffices to note that by grouping the parameters into $\tilde\theta := (\theta, \phi) \in \Theta=\R^{D}$, we can compute by Taylor expansion:
\begin{align*}
        \tilde{F}(\tilde{\theta} + t\tau, q) &= \tilde{F}(\tilde{\theta},q) + \int f(\tilde{\theta} + t\tau, z) q(z) \dd z -\int f(\tilde{\theta}, z) q(z) \dd z \\
&= \tilde{F}(\tilde{\theta},q) + \int [t \langle \tau, \nabla_{\tilde{\theta}} f(\tilde{\theta}, z) \rangle + o(t)] q(z) \dd z \\
&= \tilde{F}(\tilde\theta, q) + t \left\langle \tau, \int \nabla_{\tilde{\theta}} f(\tilde{\theta}, z) q(z) \dd z \right\rangle + o(t)
\end{align*}
where we have set $f(\tilde\theta, z):= \mathcal{L}_D(\theta, z) - \log p_\phi( x|z)$ and used $f(\tilde\theta + t\tau, z) = f(\tilde\theta, z) + t \langle \tau, \nabla_{\tilde\theta} f(\tilde\theta, z) \rangle + o(t)$. 

For \eqref{eq_app:1st_var_q}, we recall the first variation is linear, thus it suffices to compute $\delta_q \int \mathcal{L}_D(\theta, z) q(z)\dd z$ and $\delta_q\int (\log {q(z)} - \log {p_\phi( x|z)}) q(z)\dd z$ separately. Now note that:
\begin{align*}
    \int \mathcal{L}_D(\theta, z) [q(z)+tr(z)]\dd z &= \int \mathcal{L}_D(\theta, z) q(z)\dd z + t\int \mathcal{L}_D(\theta, z) r(z) \dd z
\end{align*}
and for the latter we note that $\log(z + t)(z + t) = \log(z)z + [\log(z) + 1]t + o(t)$, from which it follows:
\begin{align*}
    \int \log \left(\frac{q(z)+tr(z)}{p_\phi( x|z)}\right) [q(z)+tr(z)]\dd z 
    &= \int (\log q(z)) q(z)+ [\log q(z) + 1] tr(z) +o(t) \dd z \\
    &- \int \log p_\phi( x|z) q(z) \dd z - t\int \log p_\phi( x|z) r(z) \dd z \\
    &= \int \log \left(\frac{q(z)}{p_\phi( x|z)}\right) q(z)\dd z \\&+ t \int \log \left(\frac{q(z)}{p_\phi( x|z)}\right) r(z) \dd z + o(t)
\end{align*}
where we have used that $\int r(z)\dd z=0$ for all $r\in T\mathcal{P}_2(\mathsf{Z})$.
\end{proof}

\begin{proposition}
    Under the normalized product geometry of \Cref{app:metric}, the gradients of the averaged free energy $\tilde{F}(\theta, \phi, q^{1:M})$ with respect to $(\theta, \phi, q^{1:M})$ are given by:
    \begin{align}
    \nabla_\theta \tilde{F}(\theta, \phi, q^{1:M}) &= \frac{1}{M}\sum_{m=1}^M \int \left[\nabla_{\theta} \mathcal{L}_D(\theta, z)\right] q^m(z) \dd z \\
    \nabla_\phi \tilde{F}(\theta, \phi, q^{1:M}) &= -\frac{1}{M}\sum_{m=1}^M \int \left[\nabla_{\phi} \log p_{\phi}( x^m | z)\right] q^m(z) \dd z \\
    \nabla_{q^m} \tilde{F}(\theta, \phi, q^{1:M}) &=  \nabla_{z} \cdot \left[
        q^m(z) \nabla_{z} \left[
            \log \left(
                \frac{p_{\phi}( x^m|z)}{q^m(z)}
            \right) - \mathcal{L}_D(\theta, z)
        \right]
    \right], \quad \forall m\in[M].
\end{align}
\end{proposition}
\begin{proof}
    Recall that on the manifold $\mathcal{M}^{1:M}$, the gradient can be computed from the first variation via the metric tensor \eqref{eq:app_inv_tensor_grad}:
    \begin{equation*}
        \nabla_{(\tilde{\theta},q^{1:M})} \tilde{F}(\tilde{\theta},q^{1:M}) = G_{(\tilde{\theta}, q^{1:M})}^{-1} \delta \tilde{F}(\tilde{\theta}, q^{1:M}),
    \end{equation*}
    where $\tilde{\theta}=(\theta, \phi)$. For the parameter block, the metric is Euclidean, so the $\theta$- and $\phi$-gradients are the corresponding averaged first variations. For each distribution $q^m$, we have
    \[
        \delta_{q^m}\tilde{F}(\theta,\phi,q^{1:M}) = \frac{1}{M}\delta_q \tilde{F}^m(\theta,\phi,q^m),
    \]
    where $\delta_q \tilde{F}^m$ is given by \Cref{app:lemma_1st_var} with $x$ replaced by $x^m$ and $q$ by $q^m$. Since the inverse metric on the $m$th Wasserstein block is $M(\mathbf{G}_{q^m}^W)^{-1}$, the factor $M^{-1}$ in $\delta_{q^m}\tilde{F}$ is cancelled, yielding
    \[
        \nabla_{q^m}\tilde{F}(\theta,\phi,q^{1:M}) = (\mathbf{G}_{q^m}^W)^{-1}\delta_q \tilde{F}^m(\theta,\phi,q^m).
    \]
    Recalling that $(\mathbf{G}_q^W)^{-1}f=-\nabla_{z} \cdot (q \nabla_{z}f)$ for all $f\in C^\infty(\mathcal{P}_2(\mathsf{Z}))$ gives the desired result.
\end{proof}

\subsection{The Full Training Algorithm}
\label{app:precond_training_alg}

We detail the algorithmic considerations in \Cref{sec:pratical-alg} and provide the pseudocode for practical training in \Cref{algo:train_1}.

\subsubsection{KL Divergence Annealing} 
\label{app:kl_anneal_deriv}
We note that annealing the KL divergence with $\gamma_t$ is equivalent to applying the same weighting to the gradient of the entropy and the prior terms. In our case, we modify the gradient flow in (\ref{eq:theta_ode}--\ref{eq:flow}) as follows:
\begin{align}
    \nabla_\theta \tilde{F}(\theta_t, \phi_t, q_t^{1:M}) &= \frac{1}{M}\sum_{m=1}^M\mathbb{E}_{q_t^m(z_{0,t})}\left[\nabla_{\theta} \gamma_t\mathcal{L}_D(\theta_t, z_{0})\right], \\
    \nabla_\phi \tilde{F}(\theta_t, \phi_t, q_t^{1:M}) 
    &=-\frac{1}{M}\sum_{m=1}^M\mathbb{E}_{q_t^m(z_{0,t})}\left[\nabla_{\phi} \log p_{\phi_t}( x^m | z_{0,t})\right], \\
    \nabla_{q^m} \tilde{F}(\theta_t, \phi_t, q_t^{1:M}) &= \nabla_{z_{0}} \cdot [
        q^m(z_{0,t}) \nabla_{z_{0}}
            \log 
                p_{\phi_t}( x^m|z_{0,t})] \\
    &- \gamma_t \nabla_{z_{0}}\cdot[q^m(z_{0,t}) \nabla_{z_0} [\log {q^m(z_{0,t})}
             - \mathcal{L}_D(\theta_t, z_{0,t})
        ]
    ],  \quad\forall m\in [M]
\end{align}
which correspond to the Fokker-Planck equation of the following system of SDEs:
\begin{align}
\dd {\theta}_t &= -\frac{1}{M}\sum_{m=1}^M\mathbb{E}_{q^m_t(z_0)}\left[\nabla_{\theta} \gamma_t\mathcal{L}_D(\theta_t, Z_{0,t}^m)\right] \dd t, \\
    \dd {\phi}_t &= \frac{1}{M}\sum_{m=1}^M\mathbb{E}_{q^m_t(z_0)}\left[\nabla_{\phi} \log p_{\phi_t}( x^m | Z_{0,t}^m)\right] \dd t, \\
    \dd Z_{0,t}^m &= \nabla_{z_{0}} [\log(p_{\phi_t}( x^m|Z_{0,t}^m))-\gamma_t\mathcal{L}_D(\theta_t,Z_{0,t}^m)]dt+\sqrt{2 \gamma_t} \dd W_{t}^m, \quad\forall m\in [M]
\end{align}
Thus, we only need to adjust the amount of noise injected to the in \eqref{eq:disc_par} by setting $\sqrt{2h_z}$ to $\sqrt{2h_z \gamma_t}$ and weigh the diffusion loss $\mathcal{L}_D(\theta, z_0)$ by $\gamma_t$. Using the KL annealing scheme and re-weighted diffusion loss \eqref{app:one_sample_simplified_diffusion} discussed in \Cref{sec:pratical-alg}, we thus modify the loss in \eqref{eq:batch_loss} with:
\begin{equation}
        \hat{\mathcal{L}}(\theta_t,\phi_t,z^{1:M,1:N}_{0,t}) := \frac{1}{N |\mathcal{B}|} \sum_{(m,n)\in  \mathcal{B} \times [N]} \left[
    \gamma_t\hat{\mathcal{L}}_{\text{simple}}(\theta_t, z_{0,t}^{m,n}) - \log p_\phi ( x^m | z_{0,t}^{m,n})
    \right],\label{eq:loss_kl_anneal}
\end{equation}
where $\hat{\mathcal{L}}_{\text{simple}}$ is the simplified diffusion loss in Appendix~\ref{app:reparam_diffusion} and we take $\gamma_t = c_{KL}\min (1, t/T)$ for $T$ the number of KL warm-up steps and $c_{KL}$ a constant coefficient.

\subsubsection{Preconditioning and Momentum}
To avoid flat minima and ill-conditioning, we use an adaptive version of Langevin dynamics similar to \citet{Li2016,kim2020stochasticgradientlangevindynamics} based on Adam \citep{kingma2017adammethodstochasticoptimization}. In particular, we update each particle by running:
\begin{align}\label{eq:adam_vz_update}
z_{0,t+1}^{m,n} &\gets z_{0,t}^{m,n} + h_z G_{t}^{m,n} M_{t}^{m,n} +\sqrt{\frac{2h_z}{M}} (G_{t}^{m,n})^{1/2}W_{t}^{m,n}, \quad \forall (m,n) \in [M] \times [N]
\end{align}
where $h_z$ is the stepsize of particle updates, and we compute the moment $M_{t}^{m,n}$  and the preconditioner $G_{t}^{m,n}$ by:
\begin{subequations}\label{eq:vz_momentum}
\begin{align}
	M_{t}^{m,n} &\gets a M_{t-1}^{m,n} + {MN}(1-a) \nabla_{z_{0}^{m,n}}
\hat{\mathcal{L}}(\theta_t,\phi_t,z^{1:M,1:N}_{0,t}) \\
	V_{t}^{m,n} &\gets b V_{t-1}^{m,n} + (1-b) \text{diag}\left( \left[{MN}\nabla_{z_{0}^{m,n}  }
\hat{\mathcal{L}}(\theta_t,\phi_t,z^{1:M,1:N}_{0,t})\right]^{\otimes 2}\right)  \\
			G_{t}^{m,n} &\gets (V_{t}^{m,n} + \epsilon I)^{-1/2}
\end{align}
\end{subequations}
Here $\hat{\mathcal{L}}$ is defined as in \eqref{eq:loss_kl_anneal}, $\epsilon >0$ is a positive constant to avoid numerical instabilities, $0<a,b<1$ are hyperparameters like in Adam \citep{kingma2017adammethodstochasticoptimization} (we take $a=0.9, b=0.999$ following the default), and we used the notation $\text{diag}(v^{\otimes 2})$ to denote the diagonal matrix with the $(i,j)^{th}$ entry being $v_i^2\delta_{ij}$. In practice, we set the preconditioner on the noise in \eqref{eq:adam_vz_update} to $G_{t}^{m,n}=I$ during the first epoch for numerical stability. We also scale the noise by $|\mathcal{B}|^{-1}$ in implementation to prevent the noise from dominating the gradient. We now present the full algorithm for training.
\begin{algorithm}[htbp]
\caption{IPLD with adaptive Langevin dynamics}\label{algo:train_1}
\begin{algorithmic}[1]
\State \textbf{Inputs:} Training data points $\{ x^m\}_{m\in [M]}$,   optimizers $\texttt{Opt}_{\theta}, \texttt{Opt}_{\phi}$, Step sizes $h_\theta, h_\phi, h_z$, KL weights $\gamma_t$, Initial particles $\{z_{0,0}^{m,n}\}_{(m,n) \in [M] \times [N]}$ sampled from $\mathcal{N}(0,I)$, Initial parameters $\phi,\theta$.
\While{not converged}
    \State Sample a mini-batch of indices $\mathcal{B}\subset [M]$
    \State Compute the loss $\hat{\mathcal{L}}(\theta_t,\phi_t,z^{1:M,1:N}_{0,t})$ as in \eqref{eq:loss_kl_anneal}
    \For{$(m,n)\in \mathcal{B}\times[M]$}
        \Comment{Update the particle cloud}
	\State Compute momentum $M_{t}^{m,n}$ and preconditioner $G_{t}^{m,n}$ as in \eqref{eq:vz_momentum}
        \State Sample independent $W_{t,m}^i$ for $(i,m)\in [N]\times[M]$
        \If{$m\in \mathcal{B}$}
	\State $z_{0,t+1}^{m,n} \gets z_{0,t}^{m,n} + h_z G_{t}^{m,n} M_{t}^{m,n}$ 
    \EndIf
    \State $z_{0,t+1}^{m,n} \gets z_{0,t}^{m,n} +\sqrt{\frac{2h_z \gamma_t}{M}} (G_{t}^{m,n})^{1/2}W_{t}^{m,n}$ 
    \EndFor
    \Comment{Update model parameters}
    \State $\theta_{t+1} \gets \texttt{Opt}_\theta(h_\theta, \theta_t, \hat{\mathcal{L}})$
    \State $\phi_{t+1} \gets \texttt{Opt}_\phi(h_\phi, \phi_t, \hat{\mathcal{L}})$
    \State $t \gets t+1$
\EndWhile
\State \textbf{return} $\theta_{t}, \phi_t, z_{0,t}^{1:M,1:N}$
\end{algorithmic}
\end{algorithm}

\section{Experimental Details}
\label{app:implementaion}
For simplicity, we use a Gaussian decoder with identity covariance $p_\phi(x|z_0)=\mathcal{N}(x; g_\phi(z_0), I)$ throughout the experiments, where $g_\phi$ is the decoder network parametrized by $\phi$. We also fix the noise schedule $\{\beta_k\}_{k=1}^{K}$ as the linear schedule used in \citet{hodenoisingdiffusionprobabilistic2020} with $\beta_0=1\times 10^{-4}$ and $\beta_K =0.02$ with $K=1000$.

\subsection{Details on the synthetic experiments}\label{app:synth}
\paragraph{Data.} We create the training data by first drawing 10,000 samples from: 1)  a \gls*{GMM} with $25$ components of dimension $d_z=2$ as in \citet{boys2024tweedie,cardoso2024monte}, and 2) a concentric circle distribution also of dimension $d_z=2$ similar to that in \texttt{scikit-learn}\footnote{{\scriptsize\url{https://scikit-learn.org/stable/modules/generated/sklearn.datasets.make_circles.html}}}\citep{sklearn_api}, then we project the data into a higher-dimensional ambient space with dimension $d_x=64$ using matrices $A \in \R^{d_z \times d_x}$ with orthogonal rows. We generate the matrices using the default implementation in PyTorch of orthogonal weight initialization \citep{saxe2014exactsolutionsnonlineardynamics}. For better visualization, we set the first $2\times 2$ block to the identity for the concentric circle dataset.

\paragraph{Architecture.} For the diffusion backbones, we use a multi-layer perceptron (MLP) with 3 hidden layers, each having 128 hidden units and ReLU activation. The decoder is parametrized by a single linear layer which is equivalent to a matrix with dimension $d_z \times d_x$. For the \gls*{DiffusionVAE}, we implement the encoder also using a single linear layer equivalent to a matrix of size $d_x \times 2d_z$, where it outputs the mean and diagonal of the log-covariance matrix.

\paragraph{Training.} We train all models for 50 epochs with a batch size of 500. We use the AdamW optimizer \citep{loshchilov2019decoupledweightdecayregularization} with a learning rate of $1\times10^{-3}$ for all models. For \gls*{IPLD}, we use a version of adaptive Langevin dynamics \citep{kim2020stochasticgradientlangevindynamics} (cf. \Cref{algo:train_1}) to optimize the particles; we set the step size to $1\times 10^{-1}$ for faster convergence. The KL annealing constant $c_{KL}$ is set to $0.01$ and number of warm up steps is set to $1000$ (cf. \Cref{app:precond_training_alg}).

\paragraph{Evaluation.} We report the \gls*{MMD} \citep{jmlr:v13:gretton12a} between the generated samples and the ground truth on the \gls*{GMM} dataset. For samples $\{x_i\}_{i=1}^m \sim P$ and $\{y_i\}_{j=1}^m \sim Q$, the unbiased Monte-Carlo estimate of \gls*{MMD} is defined as:
\begin{equation*}
    \text{MMD}(P,Q) \approx \frac{1}{m(m-1)} \sum_{i=1}^{m} \sum_{\substack{j=1 \\ j \ne i}}^{m} k( x_i,  x_j)
+ \frac{1}{n(n-1)} \sum_{i=1}^{n} \sum_{\substack{j=1 \\ j \ne i}}^{n} k(y_i, y_j)
- \frac{2}{mn} \sum_{i=1}^{m} \sum_{j=1}^{n} k( x_i, y_j),
\end{equation*}
where we use the Radial Basis Function (RBF) kernel $k: \mathbb{R}^{d_x} \times \mathbb{R}^{d_x} \to \mathbb{R}_{\geq0}$ with a bandwidth $\gamma=0.1$ defined as $k({ x},{y})=\exp\left(-\gamma {{||{ x}-{y}||^{2}}}\right)$. A total of $m=n=10,000$ samples were generated via the reverse process to compute the approximation of the \gls*{MMD}. We repeat the training runs with 50 different random seeds (thus different initializations) and report the mean and standard error. The error bars shown in \Cref{fig:ldm_gmm_comparison} are the 1-standard error $\hat{\sigma}/\sqrt{50}$, where the standard deviation $\hat{\sigma}$ is estimated using \texttt{numpy}'s default implementation.

\subsection{Details on the image experiments}
\label{app:image_exp_details}
\paragraph{Architecture for \gls*{IPLD}.} We adopt the Diffusion Transformer architecture \texttt{DiT-S}, the smallest configuration from \citet{peebles2023scalablediffusionmodelstransformers}, as the backbone for our latent diffusion model. We use a patch size of $1\times 1$, as our latent space of dimension $4\times 8\times 8$ is relatively small. For the Gaussian decoder $p_\phi( x|z_0)$, we use a simplified version of \gls*{VAE}'s decoder without attention from \citet{rombachhighresolutionimagesynthesis2022}.
\begin{table}[htp]
    \centering
    \begin{tabular}{lccc}
\hline
                                         & SVHN           & CIFAR-10      & CelebA64    \\ \hline
z-shape                                  & 8 x 8 x 4      & 8 x 8 x 4     & 8 x 8 x 4   \\
Base channels                            & 128            & 128           & 128         \\
Number residual blocks per resolution    & 2              & 2             & 2           \\
Channel Multiplier                       & 1,2,3          & 1,2,3         & 1,2,2,3     \\
Batch Size                               & 128            & 128           & 16          \\
Number of Epochs                         & 400            & 400           & 200         \\
Diffusion Learning Rate                  & 1e-4           & 1e-4          & 1e-4        \\ 
Decoder Learning Rate                    & 2e-4           & 2e-4          & 2e-4        \\ 
Particle Step Size                       & 5e-2           & 5e-2          & 5e-2        \\ 
KL Warm-up Steps                         & 40000          & 40000         & 40000       \\
KL Constant Coefficient                         & 0.001          & 0.01        & 0.01       \\
EMA Decay Rate                           & 0.999          & 0.999         & 0.9999       \\
EMA Start Step                           & 40000          & 40000         & 40000       \\ 
\hline
\end{tabular}
    \vspace{0.2cm}
    \caption{An overview of \gls*{IPLD}'s settings for the experiments. Here z-shape refers to the shape of the latent vector, which in the case of \gls*{IPLD}, is the shape of a single particle. We refer the readers to \gls*{LDM} \citep{rombachhighresolutionimagesynthesis2022} and the implementation thereof for more details on the architecture.}
    \label{tab:app_arch}
\end{table}
We comment that in line with recent discussions on latent diffusions \citep{dieleman2025latents}, the spatial structure of the latent space is crucial for the successful training of a latent diffusion model. Therefore, we choose to use a decoder that can induce explicit spatial structures in the latents $z$ instead of the ones used in the \gls*{EBM} literature \citep{yu2023learning}, which compress the images to a flattened vector. However, we note our decoders have similar or fewer parameters than the implementation of \citet{yu2023learning} (both around 18 million parameters).

\paragraph{Architecture for \gls*{DiffusionVAE}.} We use the exact same diffusion backbone and decoder for our re-implementation of \gls*{VAE} with diffusion prior \citep{wehenkeldiffusionpriorsvariational2021}. For the encoder $q_\phi(z_0| x)$ of the \gls*{VAE}, we parametrize it with a Gaussian $\mathcal{N}(z_0;\mu_\phi( x), \Sigma_\phi)$, where $\Sigma_\phi$ is a diagonal matrix. The encoder's architecture also follows from \citet{rombachhighresolutionimagesynthesis2022},  which is an inverted version of the decoder.
\paragraph{Training and Evaluation.} As in the synthetic experiments, we use the AdamW optimizer 
\citep{loshchilov2019decoupledweightdecayregularization} for with a learning rate of $2\times10^{-4}$ for decoders (and encoder for \gls*{DiffusionVAE}) and $1\times 10^{-4}$ for diffusion backbones; an exponential learning rate schedule with rate $\gamma = 0.999$ was used for the decoder. For \gls*{IPLD}, we use adaptive Langevin dynamics \citep{kim2020stochasticgradientlangevindynamics} with step size $5\times 10^{-2}$ and exponential learning rate decay with $\gamma=0.995$ across all configurations following~\citet{kuntz23a}. Similar to \citet{song2020improved}, we maintain an Exponential Moving Average (EMA) of the weights of the diffusion model for evaluation; the hyperparameters for EMA are reported in \Cref{tab:app_arch}. To evaluate the generative performance, we calculate the \gls*{FID} \citep{heusel2017gan} between the true data and 50,000 generated samples using the DDIM sampler \citep{song2022denoisingdiffusionimplicitmodels} with 100 network function evaluations (NFEs).

\subsection{Warming-up with one particle}
\label{app:warm_start}
For the image experiments, we use a warm-started approach \citep{kuntz23a} for faster training. Namely, we initiate \gls*{IPLD} with a single particle $\{z_{m}\}_{m \in [M]}$ and run \Cref{algo:train_1} for 200 epochs on SVHN and CIFAR-10 (100 on CelebA64) before switching to $N>1$ particles by replicating the each particle $\{z_{m}\}_{m \in [M]}$ for $N$ times. We point out that due to the noise added in Langevin dynamics, the particles do not collapse to a single point.

\subsection{Computational Resources}
\label{app:computational_considerations}
For all experiments, we use NVIDIA GPUs. We use a single RTX 3090 for all synthetic experiments. For image experiments, the 10-particle version of \gls*{IPLD} training runs on CIFAR-10 and SVHN datasets were performed on two RTX A6000 GPUs or a single A100 GPU. The remaining image experiments were performed on a single L40S GPU. Our longest experiment takes about 27 hours on a single A100 GPU, which is around 2.5 GPU days measured in V100 using the conversion rule in \citet{rombachhighresolutionimagesynthesis2022}.

\section{Proofs of Theoretical Results}
\label{app:proofs}

\paragraph{Notation and assumptions.} We denote ${\rho}_{\theta,\phi}(\cdot)$ as the unnormalized density for:
\begin{equation}
    {p}_{\theta,\phi}(x,\cdot)=p_{\phi}(x|\cdot)p_\theta(\cdot).
\end{equation}
Similarly, we define $\tilde{\rho}_{\theta,\phi}(\cdot)$ as the unnormalized density for:
\begin{equation}
    \tilde{p}_{\theta,\phi}(x,\cdot):=p_{\theta,\phi}(x,\cdot)\exp(-\mathcal{R}(\theta,\cdot)),
\end{equation}
where $p_{\theta,\phi}(x,z_0)=p_{\theta}(z_0) p_\phi(x|z_0)$ and $\mathcal{R}(\theta, z_0) := \KL(q(z_{1:K}|z_{0})||p_{\theta}(z_{1:K}|z_{0}))$. In multi-datapoint settings, we add a superscript to denote the dependence on $x^m$ i.e. 
\begin{equation}
    \tilde\rho^m_{\theta,\phi}(\cdot):=\tilde p_{\theta,\phi}(x^m,\cdot).
\end{equation}
The $M$-datapoint version is defined as:
\begin{equation}
    \tilde{\rho}_{\theta,\phi}(z^{1:M}):=\prod_{m=1}^M\rho^m_{\theta,\phi}(z^m)=\prod_{m=1}^M \tilde{p}_{\theta,\phi}(x^m,z^m).
\end{equation}
And the joint log density is:
\begin{equation}
    \ell^m(\theta,\phi,z):=\log \rho^m_{\theta,\phi}(z)=\log p_{\phi}(x^m|z)+\log p_\theta(z).
\end{equation}
We also denote $\pi_{\theta,\phi}^m(\cdot):={p}_{\theta,\phi}(\cdot|x^m)$ as the normalized density and define similarly $\tilde \pi_{\theta,\phi}^m(\cdot):={\tilde p}_{\theta,\phi}(\cdot|x^m)$ for the tilted model.
We further denote the normalizing constant of ${p}_{\theta,\phi}(x,\cdot)$ as $A_{\theta,\phi}$ and that of $\tilde{p}_{\theta,\phi}(x,\cdot)$ as $\tilde{A}_{\theta,\phi}$. 
We recall the definition of the free energy $F(\theta, \phi,q)$ as:
\begin{equation}
    {F}(\theta, \phi, q) := \mathbb{E}_{q(z_0)}\left[
   \log \frac{q(z_{0})}{p_{\theta, \phi}( x, z_{0})}
   \right].
\end{equation}
We set the modified free energy as:
\begin{equation}
   \tilde{F}(\theta, \phi, q) := \mathbb{E}_{q(z_{0:K})}\left[ \log \frac{q(z_{0:K})}{p_{\theta, \phi}( x, z_{0:K})} \right]=F(\theta, \phi,q) + \mathbb{E}_{q(z_0)}\left[\mathcal{R}(\theta, z_0) 
   \right].
\end{equation}

The tilted free energy aggregated over multiple data points $\{x^m\}_{m=1}^M$ is defined as:
\begin{equation}
    \tilde{F}(\theta, \phi, q^{1:M}) := \frac{1}{M}\sum_{m=1}^M \tilde{F}(\theta, \phi, q^m).
\end{equation}

For clarity, we drop the subscript on $z_0$ and instead write $z$ for the latent variable in subsequent sections and we use the notations $\Theta:=\R^{d_\theta}$, $\Phi:=\R^{d_\phi}$ for the parameter spaces and $\mathcal{P}_2(\sZ)$ for the space of probability measures over the latent space $\sZ=\R^{d_z}$, which we assume to have densities with respect to the Lebesgue measure. We define the product manifold and the distance $\mathsf{d}_{\mathcal{M}}$ as in \Cref{app:metric}:
\begin{align}
    \mathcal{M} := \ \Theta \times \Phi \times \mathcal{P}_2(\sZ), \quad 
    \mathsf{d}_{\mathcal{M}}((\theta, \phi, q), (\theta', \phi', q')) := \sqrt{
    \|(\theta,\phi)-(\theta',\phi')\|^2 + \mathsf{d}_{W_2}(q, q')^2.
    }\label{eq:single_datapoint_prodmetric}
\end{align}
When $(\theta, \phi, q)$ are random variables, we overload the notation $\mathsf{d}$ to denote:
\begin{align}
\mathsf{d}((\theta, \phi, q), (\theta', \phi', q')) := \sqrt{
    \mathbb{E}[\|(\theta,\phi)-(\theta',\phi')\|^2] + \mathbb{E}[\mathsf{d}_{W_2}(q, q')^2].
    }\label{eq:random_variable_metric}
\end{align}
To extend the above to multiple datapoints, we set:
\begin{align}
    \mathcal{M}^{1:M}&:= \ \Theta \times \Phi \times \mathcal{P}_2(\sZ)^M, \\
    \mathsf{d}_{\mathcal{M}^{1:M}}((\theta, \phi, q^{1:M}), (\theta', \phi', q'^{1:M})) &:= \sqrt{
    \|(\theta,\phi)-(\theta',\phi')\|^2 + \frac{1}{M}\sum_{m=1}^M \mathsf{d}_{W_2}(q^m, q'^m)^2.\label{eq:multi_datapoint_prodmetric}
    }
\end{align}
We additionally denote $\mathcal{P}_2^1(\sZ)$ as the subset of $\mathcal{P}_2(\sZ)$ with densities differentiable almost everywhere w.r.t.~the Lebesgue measure. We use a superscript $1$ for related product spaces (\emph{e.g.}, $\mathcal{M}^1$ and $\mathcal{M}^{1:M,1}$) to indicate the restriction to $\mathcal{P}_2^1(\sZ)$.

We now restate the full assumptions for the theoretical results in the main text. 
\modelregularityrestatable*
\solutionregularityrestatable*
\stronglogconcavityrestatable*

\subsection{Full Statement and Proof of Theorem~\ref{thm:exp_convg_informal}}\label{sec_app:exp_convg_proof}
We now provide the full statement and proof of Theorem~\ref{thm:exp_convg_informal}. The strategy will be the same as \citet{caprioerrorboundsparticle2024}, where we first establish the extended log-Sobolev inequality under strong log-concavity and subsequently the extended Talagrand inequality. Our main difference from \citet{caprioerrorboundsparticle2024} is that 1) we work with the tilted model $\tilde{p}_{\theta,\phi}(x,z)$, which has an additional component arising from the diffusion loss (cf. Lemma~\ref{lem:lower_semi_continuous}), and 2) we consider the multi-datapoint setting, which leads to a different definition of the manifold and the distance thereon.

\subsubsection{Extended Log-Sobolev Inequality}
In this section, we show that strong log-concavity implies extended log-Sobolev inequality. We first define the $M$-datapoint version of the extended log-Sobolev inequality in \citet{caprioerrorboundsparticle2024}:

\begin{definition}[\gls*{xlsi}]\label{def:xLSI}
    Denote $\tilde{F}_\star:=  \inf_{(\theta, \phi,q^{1:M})\in \mathcal{M}^{1:M,1}} \tilde{F}(\theta, \phi, q^{1:M})$ as the optimum of $\tilde{F}(\theta, \phi, q^{1:M})$. We say the measure $(\tilde{\rho}_{\theta,\phi}(\dd z^{1:M}))_{(\theta,\phi)\in\Theta\times\Phi}$ satisfies the extended log-Sobolev inequality with constant $\lambda>0$ if for all $(\theta, \phi,q^{1:M}) \in \ \mathcal{M}^{1:M,1}$ we have:
    \begin{equation}
        2\lambda[\tilde{F}(\theta, \phi, q^{1:M})-\tilde{F}_\star] \leq I(\theta, \phi, q^{1:M}),
    \end{equation}
    where we define $I(\theta, \phi, q^{1:M})$ as:
    \begin{equation}\label{eq:def_fisher_info}
        \begin{aligned}
            I(\theta, \phi, q^{1:M})
            &:= \|\nabla_{(\theta, \phi)} \tilde{F}(\theta, \phi, q^{1:M})\|^2 + \frac{1}{M}\sum_{m=1}^M \int \left\| \nabla_{z} \log \left( \frac{q^{m}(z)}{\tilde{p}_{\theta,\phi}(x^m,z)}\right)  \right\|^2 q^{m}(\dd z) \\
            &= \left\|
            \frac{1}{M}\sum_{m=1}^M \int \nabla_{(\theta, \phi)} \log \tilde{p}_{\theta, \phi}(x^m,z)\, q^{m}(\dd z)
            \right\|^2 + \frac{1}{M}\sum_{m=1}^M \int \left\| \nabla_{z} \log \left( \frac{q^{m}(z)}{\tilde{p}_{\theta,\phi}(x^m,z)}\right)  \right\|^2 q^{m}(\dd z).
        \end{aligned}
    \end{equation}
\end{definition}

Using the functional $I(\theta, \phi, q^{1:M})$ defined in~\eqref{eq:def_fisher_info}, we can state the following extension of de Bruijn's identity:
\begin{proposition}[de Bruijn's Identity]\label{prop:de_bruijin}
    Under Assumption~\Cref{assump:solution_reg}, we have:
    \begin{equation}\label{eq:de_bruijin}
        \frac{\dd}{\dd t} \tilde{F}(\theta_t, \phi_t, q_t^{1:M}) = -I(\theta_t, \phi_t, q_t^{1:M}), \qquad \forall t>0.
    \end{equation}
\end{proposition}
\begin{proof}
    By definition,
    \begin{align*}
        \tilde{F}(\theta_t,\phi_t,q_t^{1:M})
        = \frac{1}{M}\sum_{m=1}^M \left[
            \int \log\left(\frac{q_t^m(z)}{\tilde{p}_{\theta_t,\phi_t}(x^m,z)}\right) q_t^m(\dd z)
        \right].
    \end{align*}
    Under Assumption~\Cref{assump:solution_reg}, we can differentiate under the integral sign and the flow satisfies
    \begin{align*}
        \frac{\partial_t q_t^m(z)}{\partial t} &= \nabla_z \cdot \left[
            q_t^m(z)\nabla_z \log \left(\frac{q_t^m(z)}{\tilde{p}_{\theta_t,\phi_t}(x^m,z)}\right)
        \right], \\
        (\dot{\theta}_t,\dot{\phi}_t) &= \frac{1}{M}\sum_{m=1}^M \int \nabla_{(\theta,\phi)} \log \tilde{p}_{\theta_t,\phi_t}(x^m,z)\, q_t^m(\dd z).
    \end{align*}
    Using integration by parts, for each $m$ we obtain
    \begin{align*}
        \frac{\dd}{\dd t} \int \log(q_t^m(z)) q_t^m(\dd z)
        &= \int (\log(q_t^m(z))+1)\, \frac{\partial_t q_t^m(z)}{\partial t}\dd z \\
        &= - \int \left\langle \nabla_z \log(q_t^m(z)), \nabla_z \log \left(\frac{q_t^m(z)}{\tilde{p}_{\theta_t,\phi_t}(x^m,z)}\right) \right\rangle q_t^m(\dd z)
    \end{align*}
    and
    \begin{align*}
        \frac{\dd}{\dd t} \int \log(\tilde{p}_{\theta_t,\phi_t}(x^m,z)) q_t^m(\dd z)
        &= \int \left\langle \nabla_{(\theta,\phi)} \log \tilde{p}_{\theta_t,\phi_t}(x^m,z), (\dot{\theta}_t,\dot{\phi}_t) \right\rangle q_t^m(\dd z) \\
        &\quad - \int \left\langle \nabla_z \log(\tilde{p}_{\theta_t,\phi_t}(x^m,z)), \nabla_z \log \left(\frac{q_t^m(z)}{\tilde{p}_{\theta_t,\phi_t}(x^m,z)}\right) \right\rangle q_t^m(\dd z).
    \end{align*}
    Therefore,
    \begin{align*}
        \frac{\dd}{\dd t} \tilde{F}(\theta_t,\phi_t,q_t^{1:M})
        &= - \|(\dot{\theta}_t,\dot{\phi}_t)\|^2 - \frac{1}{M}\sum_{m=1}^M \int \left\| \nabla_z \log \left(\frac{q_t^m(z)}{\tilde{p}_{\theta_t,\phi_t}(x^m,z)}\right) \right\|^2 q_t^m(\dd z) \\
        &= - I(\theta_t,\phi_t,q_t^{1:M}),
    \end{align*}
    where the last equality follows from the definition of $I$ in~\eqref{eq:def_fisher_info}.
\end{proof}

We need a few auxiliary results that are extensions of those in \citet{caprioerrorboundsparticle2024}.
\begin{lemma}[Geodesics on $\mathcal{M}^{1:M}$]\label{lem:geodesic}
    A curve $\gamma(t) : t \in [0, 1] \mapsto \mathcal{M}^{1:M}$ is a geodesic if and only if $\gamma(t) = (\gamma_{\theta}(t), \gamma_\phi (t), \gamma_{q^1}(t), \ldots, \gamma_{q^M}(t))$, where $\gamma_{\theta}$ and $\gamma_\phi$ are geodesics in the Euclidean parameter space and each $\gamma_{q^m}$ is a geodesic in $(\mathcal{P}_2(\sZ), \mathrm{d}_{W_2})$. In particular, if $\gamma(t)$ is a geodesic in $\mathcal{M}^{1:M}$ connecting $(\theta, \phi, q^{1:M})$ and $(\theta', \phi', (q')^{1:M})$ then 
    \begin{align*}
        \gamma_{\theta}(t) &= (1-t)\theta + t\theta', \\
        \gamma_{\phi}(t) &= (1-t)\phi + t\phi', \\
        \gamma_{q^m}(t) &= (h_t)_{\#}\varrho^m, \quad m=1,\ldots,M,
    \end{align*}
    where $\varrho^m$'s are Wasserstein-2 optimal transport plans for $(q^m, (q')^m)$ and $h_t(z,z') = (1-t)z + tz'$. Furthermore, if $q^m$'s have densities with respect to the Lebesgue measure, then we can also write:
    \begin{align*}
        \gamma_{q^m}(t) = ((1-t)\mathrm{id} + t\nabla_z f^m)_{\#} q^m, \quad \forall m \in [M],
    \end{align*}
    for some convex function $f^m$.
\end{lemma}
\begin{proof}
    The first claim follows from the definition of the product metric~\eqref{eq:multi_datapoint_prodmetric}. The second claim is a result of the characterization of geodesics (cf.~\citet[Theorem 5.27]{santambrogio2015optimal}). The last claim follows from Brenier's theorem (cf.~\citet[Theorem 1.17]{santambrogio2015optimal}).
\end{proof}

\begin{lemma}[Geodesic convexity of $\tilde{F}$]\label{lem:geodesic_convexity}
    Under Assumption~\Cref{assump:strong_log_concave}, the tilted free energy $\tilde{F}(\theta, \phi, q^{1:M})$ is $\lambda$-geodesically convex on $\mathcal{M}^{1:M}$, that is, for any pair $(\theta, \phi, q^{1:M})$ and $(\theta', \phi', (q')^{1:M})$ in $\mathcal{M}^{1:M}$ and any geodesic $\gamma(t)$ connecting them, we have:
    \begin{align*}
        \tilde{F}(\gamma(t)) &\leq (1-t)\tilde{F}(\theta, \phi, q^{1:M}) + t\tilde{F}(\theta', \phi', (q')^{1:M}) 
        - \frac{\lambda t(1-t)}{2}\mathsf{d}_{\mathcal{M}^{1:M}}((\theta, \phi, q^{1:M}), (\theta', \phi', (q')^{1:M}))^2.
    \end{align*}
\end{lemma}
\begin{proof}
    First recall that $\tilde{F}(\theta, \phi, q^{1:M})= M^{-1} \sum_{m=1}^M \mathbb{E}_{q^m(z)}[\log(q^m(z))-\tilde{\ell}^m(\theta,\phi,z)]$.
    We note that the negative entropy $q^m \mapsto \int \log(q^m(z)) q^m(\dd z)$ is geodesically convex on $\mathcal{P}_2(\sZ)$ (cf.~\citet[Theorem 7.28]{santambrogio2015optimal}) and thus the average $q^{1:M} \mapsto M^{-1}\sum_{m=1}^M \int \log(q^m(z)) q^m(\dd z)$ is geodesically convex on $\mathcal{P}_2(\sZ)^M$. By an argument similar to~\citet[Lemma 21]{caprioerrorboundsparticle2024}, we can show the map $V: (\theta, \phi, z) \mapsto -M^{-1}\sum_{m=1}^M \int \tilde{\ell}^m(\theta,\phi, z) q^m(\dd z)$ is $\lambda$-strongly convex along the geodesic $\gamma(t)$:
    \begin{align*}
        V(\gamma(t)) &= - \int \frac{1}{M}\sum_{m=1}^{M} \tilde{\ell}^m(\gamma_\theta(t), \gamma_\phi(t), z) \gamma_{q^m}(t)(\dd z) \\
        &=-\int \frac{1}{M}\sum_{m=1}^{M} \tilde{\ell}^m((1-t)\theta + t\theta', (1-t)\phi + t\phi', (1-t)z + tz') \varrho^m(\dd z, \dd z') \\
        &\leq -\int \frac{1}{M}\sum_{m=1}^{M} (1-t) \tilde{\ell}^m(\theta, \phi, z) + t\tilde{\ell}^m(\theta', \phi', z') \varrho^m(\dd z, \dd z') \\
        &\quad - \frac{\lambda t(1-t)}{2} \int \frac{1}{M}\sum_{m=1}^M \|(\theta, \phi, z) - (\theta', \phi', z')\|^2 \varrho^m(\dd z, \dd z') \\
        &=(1-t)V(\theta, \phi, q^{1:M}) + tV(\theta', \phi', (q')^{1:M}) - \frac{\lambda t(1-t)}{2} \mathsf{d}_{\mathcal{M}^{1:M}}((\theta, \phi, q^{1:M}), (\theta', \phi', (q')^{1:M}))^2,
    \end{align*}
    where we have used Lemma~\ref{lem:geodesic} in the second equality and the strong log-concavity in the inequality. The conclusion follows from the definition of the product metric~\eqref{eq:multi_datapoint_prodmetric}.
\end{proof}
We similarly provide an extension to~\citet[Lemma~22]{caprioerrorboundsparticle2024}:
\begin{lemma}\label{lem:lower_deriv}
    Let $\gamma(t)$ be a geodesic in $\mathcal{M}^{1:M,1}$ connecting $(\theta, \phi, q^{1:M})$ and $(\theta', \phi', (q')^{1:M})$. Then we have:
    \begin{align*}
        \liminf_{t \to 0^+} \frac{\tilde{F}(\gamma(t)) - \tilde{F}(\gamma(0))}{t} &\geq \frac{1}{M} \sum_{m=1}^M \left\langle \dot{\gamma}_{q^m}(0), \nabla_{q^m} \tilde{F}(\theta, \phi, q^{m}) \right\rangle_{q^m} + \langle 
        (\theta', \phi') - (\theta, \phi), \nabla_{(\theta, \phi)} \tilde{F}(\theta, \phi, q^{1:M})
        \rangle
    \end{align*}
\end{lemma}
\begin{proof}
    The proof is an adaptation of~\citet[Lemma 22]{caprioerrorboundsparticle2024} by noting that the computation can be done separately for each $m$.
\end{proof}
\begin{theorem}[Strong log-concavity $\implies$ xLSI]\label{thm:sLC_implies_xLSI}
    Suppose Assumptions~\Cref{assump:model_reg} and~\Cref{assump:strong_log_concave} hold, then the family of measures $(\tilde{\rho}_{\theta,\phi}(\dd z^{1:M}))_{(\theta,\phi)\in\Theta\times\Phi}$ satisfies the extended log-Sobolev inequality with constant $\lambda>0$.
\end{theorem}
\begin{proof}
    The proof extends that of~\citet[Theorem 6]{caprioerrorboundsparticle2024} to the product manifold $\Theta \times \Phi \times \mathcal{P}_2(\sZ)^M$. Let $\gamma(t)$ be a geodesic in $\mathcal{M}^{1:M}$ connecting $(\theta, \phi, q^{1:M})$ and $(\theta', \phi', (q')^{1:M})$. By Lemma~\ref{lem:geodesic_convexity}, we have:
    \begin{align*}
        \liminf_{t \to 0^+} \frac{\tilde{F}(\gamma(t)) - \tilde{F}(\gamma(0))}{t} &\leq \tilde{F}(\theta', \phi', (q')^{1:M}) - \tilde{F}(\theta, \phi, q^{1:M}) - \frac{\lambda}{2}\mathsf{d}_{\mathcal{M}^{1:M}}((\theta, \phi, q^{1:M}), (\theta', \phi', (q')^{1:M}))^2.
    \end{align*}
    Setting $(\theta', \phi', (q')^{1:M})=(\theta_\star, \phi_\star, (q_\star')^{1:M})$ to be a minimizer of $\tilde{F}$ and using the previous Lemma~\ref{lem:lower_deriv}, we obtain:
    \begin{align*}
        \tilde{F}(\theta, \phi, q^{1:M}) - \tilde{F}_\star &\leq -\frac{1}{M} \sum_{m=1}^M \left\langle \dot{\gamma}_{q^m}(0), \nabla_q { \tilde{F}(\theta, \phi, q^m)} \right\rangle_{q^m} - \langle 
        (\theta_\star, \phi_\star) - (\theta, \phi), \nabla_{(\theta, \phi)} \tilde{F}(\theta, \phi, q^{1:M})
        \rangle \\
        &\quad - \frac{\lambda}{2}\mathsf{d}_{\mathcal{M}^{1:M}}((\theta, \phi, q^{1:M}), (\theta_\star, \phi_\star, q_\star^{1:M}))^2 \\
        &\leq \frac{1}{M} \sum_{m=1}^M \mathsf{d}_{W_2}(q^m, q_\star^m) \left\| \nabla_q {\tilde{F}(\theta, \phi, q^m)} \right\|_{q^m} + \|(\theta_\star, \phi_\star) - (\theta, \phi)\| \|\nabla_{(\theta, \phi)} \tilde{F}(\theta, \phi, q^{1:M})\| \\
        &\quad - \frac{\lambda}{2}\mathsf{d}_{\mathcal{M}^{1:M}}((\theta, \phi, q^{1:M}), (\theta_\star, \phi_\star, q_\star^{1:M}))^2,
    \end{align*}
    where we have used the definition of the inner product on the tangent space $T_q\mathcal{P}_2^1(\sZ)$ and Lemma~\ref{lem:geodesic}:
    \begin{equation*}
        \|\dot{\gamma}_{q^m}(0)\|_{q^m}=\|\nabla_z f^m - \mathrm{id}\|_{L^2(q^m)}=\mathsf{d}_{W_2}(q^m, q_\star^m),
    \end{equation*}
    and the Cauchy-Schwarz inequality. One more application of the Cauchy-Schwarz inequality yields:
    \begin{align*}
        \tilde{F}(\theta, \phi, q^{1:M}) - \tilde{F}_\star &\leq \left[\frac{1}{M}\sum_{m=1}^M \mathsf{d}_{W_2}(q^m, q_\star^m)^2 + \|(\theta, \phi) - (\theta_\star, \phi_\star)\|^2\right]^{1/2} \\
        &\quad \times \left[\frac{1}{M}\sum_{m=1}^M \left\| \nabla_{q} { \tilde{F}(\theta, \phi, q^m)} \right\|_{q^m}^2 + \|\nabla_{(\theta, \phi)} \tilde{F}(\theta, \phi, q^{1:M})\|^2\right]^{1/2} \\
        &\quad - \frac{\lambda}{2}\mathsf{d}_{\mathcal{M}^{1:M}}((\theta, \phi, q^{1:M}), (\theta_\star, \phi_\star, q_\star^{1:M}))^2 \\
        &= {\mathsf{d}_{\mathcal{M}^{1:M}}((\theta, \phi, q^{1:M}), (\theta_\star, \phi_\star, q_\star^{1:M}))} \sqrt{I(\theta, \phi, q^{1:M})} - \frac{\lambda}{2}\mathsf{d}_{\mathcal{M}^{1:M}}((\theta, \phi, q^{1:M}), (\theta_\star, \phi_\star, q_\star^{1:M}))^2 \\
        &\leq \frac{1}{2\lambda} I(\theta, \phi, q^{1:M}),
    \end{align*}
    where we have upper bounded the first term using Young's inequality $ab \leq a^2/(2\lambda) + \lambda b^2/2$ in the last step.
\end{proof}

\subsubsection{Extended Talagrand-type Inequality}
We now show that \gls*{xlsi} implies an extended Talagrand-type inequality. Throughout this section, we assume Assumptions~\Cref{assump:model_reg} and \Cref{assump:solution_reg} hold. We first provide the $M$-datapoint version of extended Talagrand in \citet{caprioerrorboundsparticle2024}.
\begin{definition}[Extended Talagrand-type Inequality]\label{defn:xTalagrand}
    The family of measures $(\tilde{\rho}_{\theta,\phi}(\dd z^{1:M}))_{(\theta,\phi)\in\Theta\times\Phi}$ satisfy the extended Talagrand-type inequality with constant $\lambda>0$ if for all $(\theta, \phi,q^{1:M}) \in \ \mathcal{M}^{1:M}$:
    \begin{align}
        2[\tilde{F}(\theta, \phi, q^{1:M})-\tilde{F}_\star] &\geq \lambda\inf_{(\theta, \phi,q^{1:M}) \in \ \mathcal{M}^{1:M}} \mathsf{d}((\theta, \phi, q^{1:M}), \mathcal{M}^{1:M}_\star)^2
    \end{align}
    where $\mathcal{M}^{1:M}_\star:=\arg\min_{(\theta, \phi,q^{1:M}) \in \ \mathcal{M}^{1:M}} \tilde{F}(\theta, \phi, q^{1:M})$ is the optimal set of $\tilde{F}$.
\end{definition}

 The proof hinges on a few auxiliary results adapted from \citet{caprioerrorboundsparticle2024}.
\begin{lemma}\label{lem:distance_deriv}
    Under Assumption~\Cref{assump:solution_reg}, we have:
    \begin{equation*}
        \frac{\dd}{\dd t} \mathsf{d}_{\mathcal{M}^{1:M}}((\theta_t, \phi_t, q_t^{1:M}), (\theta, \phi, q^{1:M})) \leq \sqrt{I(\theta_t, \phi_t, q_t^{1:M})}, \quad \forall t>0,
    \end{equation*}
    where $I(\theta, \phi, q^{1:M})$ is defined in~\eqref{eq:def_fisher_info}.
\end{lemma}
\begin{proof}
    For $(\theta, \phi, q^{1:M}) \in \mathcal{M}^{1:M}$, we define the velocity field $v_t^{1:M} = (v_t^1, \ldots, v_t^M)$ as:
    \begin{align*}
        v_t^m(z) = -\nabla_z \log \left( \frac{q_t^m(z)}{\tilde{\rho}_{\theta_t, \phi_t}^m(z)} \right), \quad \forall m \in [M].
    \end{align*}
    Using the proof to~\citet[Lemma 16]{caprioerrorboundsparticle2024}, for each $m$, we have:
    \begin{align*}
        \frac{\dd}{\dd t} \mathsf{d}_{W_2}(q_t^m, q^m)^2 \leq 2 \mathsf{d}_{W_2}(q_t^m, q^m)\sqrt{\int \|v_t^m(z)\|^2 q_t^m(\dd z)}.
    \end{align*}
    By the Cauchy-Schwarz inequality, we have:
    \begin{align*}
        \frac{\dd}{\dd t} \frac{1}{M} \sum_{m=1}^M \mathsf{d}_{W_2}(q_t^m, q^m)^2 &\leq \frac{2}{M} \sum_{m=1}^M \mathsf{d}_{W_2}(q_t^m, q^m)\sqrt{\int \|v_t^m(z)\|^2 q_t^m(\dd z)}.
    \end{align*}
    Combining this with the definition of the Euclidean counterpart:
    \begin{align*}
        \frac{\dd}{\dd t} \|(\theta_t, \phi_t) - (\theta, \phi)\|^2 = 2\left\langle \frac{\dd}{\dd t} (\theta_t, \phi_t), (\theta_t, \phi_t) - (\theta, \phi) \right\rangle \leq 2\left\|\nabla_{(\theta, \phi)} \tilde{F}(\theta_t, \phi_t, q_t^{1:M})\right\| \|(\theta_t, \phi_t) - (\theta, \phi)\|,
    \end{align*}
    we obtain with an application of the Cauchy-Schwarz inequality similar to~\Cref{thm:sLC_implies_xLSI}:
    \begin{align*}
        \frac{\dd}{\dd t}\frac{1}{2} \mathsf{d}_{\mathcal{M}^{1:M}}((\theta_t, \phi_t, q_t^{1:M}), (\theta, \phi, q^{1:M}))^2 &\leq  \left(\frac{1}{M}\sum_{m=1}^M \mathsf{d}_{W_2}(q_t^m, q^m)^2\right)^{1/2}\left(\frac{1}{M}\sum_{m=1}^M \int \|v_t^m(z)\|^2 q_t^m(\dd z)\right)^{1/2} \\
        &\qquad + \|(\theta_t, \phi_t) - (\theta, \phi)\| \|\nabla_{(\theta, \phi)} \tilde{F}(\theta_t, \phi_t, q_t^{1:M})\| \\
        &\leq \sqrt{I(\theta_t, \phi_t, q_t^{1:M})} \mathsf{d}_{\mathcal{M}^{1:M}}((\theta_t, \phi_t, q_t^{1:M}), (\theta, \phi, q^{1:M})),
    \end{align*}
    which implies the conclusion.
\end{proof}

\begin{lemma}\label{lem:lower_semi_continuous}
    The tilted free energy $\tilde{F}(\theta, \phi, q^{1:M})$ is lower semi-continuous on $\mathcal{M}^{1:M,1}$.
\end{lemma}
\begin{proof}
    We can re-write $\tilde{F}(\theta, \phi, q^{1:M}) = M^{-1}\sum_{m=1}^M {F}(\theta, \phi, q^m)+\mathbb{E}_{q^m}[\mathcal{R}(\theta, z)]$, where we recall $\mathcal{R}(\theta, z_0) = \KL\big(q(z_{1:K}|z_{0}) \|p_{\theta}(z_{1:K}|z_{0})\big) \geq 0$.
    We note that it suffices to show that $U: (\theta, q) \mapsto \mathbb{E}_{q}[\mathcal{R}(\theta, z)] = \int \mathcal{R}(\theta, z) q(\dd z)$ is lower semi-continuous on $\Theta \times \mathcal{P}_2(\sZ)$ (hence $\Theta \times \mathcal{P}_2^1(\sZ)$), since~\citet[Lemma 18]{caprioerrorboundsparticle2024} shows that the untilted free energy ${F}(\theta, \phi, q)$ is lower semi-continuous on $\mathcal{M}^1$. The conclusion thus follows from the definition of the space $(\mathcal{M}^{1:M,1}, \mathsf{d}_{\mathcal{M}^{1:M}})$.

    To show $U$ is lower semi-continuous, let $\{(\theta_n, q_n)\}_{n\in\mathbb{N}}$ be a sequence in $\Theta \times \mathcal{P}_2(\sZ)$ converging to some $(\theta_\infty, q_\infty) \in \Theta \times \mathcal{P}_2(\sZ)$ in the topology induced by $\mathsf{d}_{\mathcal{M}}$. The convergence in $\mathsf{d}_{W_2}$ implies weak convergence of $q_n \rightharpoonup q_\infty$ \citep{figalli2021invitation} and we have $\theta_n \to \theta_\infty$.
    Since $\mathcal{R}$ is non-negative and continuous in both arguments, by applying Portmanteau theorem on the product measures $\mu_n:=\delta_{\theta_n}(\dd \theta) \otimes q_n(\dd z)$,
    we have:
    \begin{align*}
        \liminf_{n\to\infty} U(\theta_n, q_n) &= \liminf_{n\to\infty} \int \mathcal{R}(\theta_n, z) \mu_n(\dd \theta, \dd z)\\
        &\geq\int \mathcal{R}(\theta, z) \mu_\infty(\dd \theta, \dd z) = U(\theta_\infty, q_\infty),
    \end{align*}
    which concludes the proof.
\end{proof}

\begin{lemma}\label{lem:cauchy}
    Denote $\mathcal{M}_\star := \arg\min_{(\theta, \phi, q^{1:M}) \in \mathcal{M}^{1:M}} \tilde{F}(\theta, \phi, q^{1:M})$ as the set of minimizers of $\tilde{F}$ and $\tilde{F}_\star := \inf_{(\theta, \phi, q^{1:M}) \in \mathcal{M}^{1:M}} \tilde{F}(\theta, \phi, q^{1:M})$ as the optimal value. Under the extended log-Sobolev inequality, for a Cauchy sequence $\{(\theta_n, \phi_n, q_n^{1:M})\}_{n\in \mathbb{N}}$ in $\mathcal{M}^{1:M,1}$, there exists an increasing sequence of $t_n \to +\infty$ such that:
    \begin{equation*}
        (\theta_{t_n}, \phi_{t_n}, q_{t_n}^{1:M}) \to (\theta_\star, \phi_\star, q_\star^{1:M}) \in \mathcal{M}_\star, \quad \text{as} \ n \to +\infty,
    \end{equation*}
    for some $(\theta_\star, \phi_\star, q_\star^{1:M}) \in \mathcal{M}_\star$ in the topology induced by $\mathsf{d}_{\mathcal{M}^{1:M}}$ on $\mathcal{M}^{1:M}$.
\end{lemma}
\begin{proof}
    Since both the Euclidean space and $\mathcal{P}_2(\sZ)$ are complete metric spaces \citep[Theorem 6.18]{villani2008optimal}, the product space $(\mathcal{M}^{1:M}, \mathsf{d}_{\mathcal{M}^{1:M}})$ is also complete. Therefore, the sequence $(\theta_\infty, \phi_\infty, q_\infty^{1:M})$ is also Cauchy in $\mathcal{M}^{1:M}$ and converges to some limit $(\theta_\infty, \phi_\infty, q_\infty^{1:M}) \in \mathcal{M}^{1:M}$. By the lower semi-continuity of $\tilde{F}$, we have:
    \begin{align*}
        \tilde{F}_\star \leq \tilde{F}(\theta_\infty, \phi_\infty, q_\infty^{1:M}) \leq \liminf_{n\to\infty} \tilde{F}(\theta_{t_n}, \phi_{t_n}, q_{t_n}^{1:M}) = \tilde{F}_\star,
    \end{align*}
    where the first inequality is by the optimality of $\tilde{F}_\star$ and the equality follows from the fact the exponential convergence induced by \gls*{xlsi} (cf. the rightmost inequality of~\Cref{thm:exp_convg_appendix} below).
\end{proof}

\begin{theorem}[\gls*{xlsi} $\implies$ \gls*{xt2i}]\label{thm:xLSI_implies_xTalagrand}
    Suppose the family of measures $(\tilde{\rho}_{\theta,\phi}(\dd z^{1:M}))_{(\theta,\phi)\in\Theta\times\Phi}$ satisfies the extended log-Sobolev inequality with constant $\lambda>0$. Then it also satisfies the extended Talagrand-type inequality with the same constant $\lambda>0$.
\end{theorem}
\begin{proof}
    Using Lemma~\ref{lem:distance_deriv} and the \gls*{xlsi}, an argument similar to that in~\citet[Theorem 4]{caprioerrorboundsparticle2024} yields the following sequence of inequalities for any $(\theta, \phi, q^{1:M}) \in \mathcal{M}^{1:M}$:
    \begin{align*}
        \frac{\dd}{\dd t} \mathsf{d}_{\mathcal{M}^{1:M}}((\theta_t, \phi_t, q_t^{1:M}), (\theta, \phi, q^{1:M})) \leq \sqrt{I(\theta_t, \phi_t, q_t^{1:M})} &\leq \frac{I(\theta_t, \phi_t, q_t^{1:M})}{\sqrt{2\lambda[\tilde{F}(\theta_t, \phi_t, q_t^{1:M}) - \tilde{F}_\star]}}. 
    \end{align*}
    Using de Bruijn's identity~\eqref{eq:de_bruijin}, we have:
    \begin{align*}
        \frac{\dd}{\dd t} \mathsf{d}_{\mathcal{M}^{1:M}}((\theta_t, \phi_t, q_t^{1:M}), (\theta, \phi, q^{1:M})) &\leq -\frac{\dd}{\dd t} \sqrt{\frac{2[\tilde{F}(\theta_t, \phi_t, q_t^{1:M}) - \tilde{F}_\star]}{\lambda}}.
    \end{align*}
    For any interval $(t,t')$, integrating yields:
    \begin{align}
        &\mathsf{d}_{\mathcal{M}^{1:M}}((\theta_{t'}, \phi_{t'}, q_{t'}^{1:M}), (\theta, \phi, q^{1:M})) - \mathsf{d}_{\mathcal{M}^{1:M}}((\theta_t, \phi_t, q_t^{1:M}), (\theta, \phi, q^{1:M})) \nonumber \\
        &\leq \sqrt{\frac{2[\tilde{F}(\theta_t, \phi_t, q_t^{1:M}) - \tilde{F}_\star]}{\lambda}} - \sqrt{\frac{2[\tilde{F}(\theta_{t'}, \phi_{t'}, q_{t'}^{1:M}) - \tilde{F}_\star]}{\lambda}}.\label{eq:cauchy_seq}
    \end{align}
    We can thus construct a Cauchy sequence $\{(\theta_{t_n}, \phi_{t_n}, q_{t_n}^{1:M})\}_{n\in \mathbb{N}}$ in $\mathcal{M}^{1:M,1}$ from~\eqref{eq:cauchy_seq} for some increasing sequence $t_n \to +\infty$. By the previous Lemma~\ref{lem:cauchy}, we have the limit point $(\theta_\infty, \phi_\infty, q_\infty^{1:M}) \in \mathcal{M}_\star$. Setting $(t,t')=(0,t_n)$ in~\eqref{eq:cauchy_seq} and letting $n \to +\infty$, we have:
    \begin{align*}
        \mathsf{d}_{\mathcal{M}^{1:M}}((\theta_0, \phi_0, q_0^{1:M}), (\theta_\infty, \phi_\infty, q_\infty^{1:M})) \leq \sqrt{\frac{2[\tilde{F}(\theta_0, \phi_0, q_0^{1:M}) - \tilde{F}_\star]}{\lambda}},
    \end{align*}
    whence the conclusion follows by noting the distance function is continuous and the infimum is attained.
\end{proof}

We now state the full-version of \Cref{thm:exp_convg_informal}.

\begin{theorem}\label{thm:exp_convg_appendix}
    Suppose Assumptions~\Cref{assump:model_reg}-\Cref{assump:strong_log_concave} are satisfied. Then, $\tilde{\ell}$ has a unique maximizer $(\theta_\star, \phi_\star)$ and the flow converges exponentially fast to it: for some $\lambda>0$ independent of $M$, 
\begin{equation}\label{eq:app_thm_exp}
        \mathsf{d}_{\mathcal{M}^{1:M}}((\theta_t,\phi_t,q_t^{1:M}),(\theta_\star,\phi_\star,q_{\star}^{1:M}))
        \leq
        \sqrt{\frac{2[\tilde{F}(\theta_t, \phi_t, q_t^{1:M}) - \tilde{F}_\star]}{\lambda}}
        \leq
        \sqrt{\frac{2[\tilde{F}(\theta_0, \phi_0, q_0^{1:M})-\tilde{F}_\star]}{\lambda}}e^{-\lambda t}, \quad \forall t>0;
\end{equation}
where $\tilde{F}_\star:=\inf_{(\theta,\phi,q^{1:M})\in \mathcal{M}^{1:M}}\tilde{F}(\theta,\phi,q^{1:M})$ and $\vert\vert\cdot\vert\vert$ denotes the Euclidean norm.
\end{theorem}

\begin{proof}

Under Assumption~\Cref{assump:strong_log_concave}, each map $(\theta,\phi,z)\mapsto \log \tilde{p}_{\theta,\phi}(x^m,z)$ is $\lambda$-strongly concave. Hence the product density
$$\bar{p}_{\theta,\phi}(x^{1:M},z^{1:M}):=\prod_{m=1}^M\tilde{p}_{\theta,\phi}(x^m,z^m)$$
has joint log-density
$$\bar{\ell}(\theta,\phi,z^{1:M}):=\log \bar{p}_{\theta,\phi}(x^{1:M},z^{1:M})=\sum_{m=1}^M\log\tilde{p}_{\theta,\phi}(x^m,z^m),$$
which is strictly concave in $((\theta,\phi),z^{1:M})$. Applying~\citet[Theorem 4]{kuntz23a} with parameter $u=(\theta,\phi)$ and latent variable $z^{1:M}$ yields that the marginal log-likelihood
$$\bar{\ell}(\theta,\phi):=\log \int \bar{p}_{\theta,\phi}(x^{1:M},z^{1:M}) \dd z^{1:M}=\sum_{m=1}^M\log\tilde{p}_{\theta,\phi}(x^m)$$
has a unique maximizer. Since $\tilde{\ell}(\theta,\phi)=M^{-1}\bar{\ell}(\theta,\phi)$, the same pair $(\theta_\star,\phi_\star)$ is the unique maximizer of $\tilde{\ell}$.

Under the log-concavity assumption, the family of distributions $(\tilde{\rho}_{\theta,\phi}(\dd z^{1:M}))_{(\theta,\phi)\in\Theta\times\Phi}$ satisfies the extended log-Sobolev inequality (cf. Definition~\ref{def:xLSI}) using \Cref{thm:sLC_implies_xLSI}.

The rightmost inequality in~\eqref{eq:app_thm_exp} then follows from a combination of de Bruijn's identity~\eqref{eq:de_bruijin}, the extended log-Sobolev inequality, and Gr\"onwall's lemma. 

The leftmost inequality follows from a combination of the uniquness of minimizers and \Cref{thm:xLSI_implies_xTalagrand}. 
\end{proof}
\begin{remark}\label{remark:star_notation}
    Since $\tilde{\ell}$ has a unique maximizer $(\theta_\star,\phi_\star)$ and Proposition~\ref{prop:ftildemins} identifies the corresponding minimizer of $\tilde{F}$, we can write $\tilde{F}_\star =\tilde{F}(\theta_\star,\phi_\star,\tilde\pi_{\theta_\star,\phi_\star}^{1:M})$ and $q_\star^{1:M}=\tilde\pi_{\theta_\star,\phi_\star}^{1:M}$, where we recall $\tilde\pi_{\theta,\phi}(\cdot)=\tilde{p}_{\theta,\phi}(\cdot|x)$ is the true posterior distribution.
\end{remark}

\subsection{Proof of Theorem~\ref{thrm:error-bound}}\label{sec_app:discretize_proof}
We now provide the proof to Theorem~\ref{thrm:error-bound}. The proof will be based on the spatial and temporal discretization error bounds established in \citet{caprioerrorboundsparticle2024}, which are combined with the exponential convergence result in Theorem~\ref{thm:exp_convg_appendix}. In what follows, we will use the notation \[\tilde\theta:=(\theta, \phi)\]s
to denote the concatenation of the diffusion prior and decoder parameters as in Appendix~\ref{app:grad_deriv}.

First, we recall the additional assumption.
\begin{trivlist}
\item[\hskip\labelsep \textbf{Assumption}~\ref{assump:lipschitz_grad} (Lipschitz gradient).]\itshape
The log-likelihood $\tilde\ell^m(\tilde\theta, z):= \log p_{\tilde\theta}(x^m, z)$ is differentiable and its gradient $\nabla \tilde\ell^m := (\nabla_{\tilde\theta}\tilde\ell^m, \nabla_{z_0}\tilde\ell^m)$ is $L$-Lipschitz for some $L > 0$, that is, for all $(\tilde\theta, z), (\tilde\theta', z') \in (\Theta \times \Phi) \times \sZ$:
\begin{align*}
    \|\nabla \tilde\ell^m(\tilde\theta, z) - \nabla \tilde\ell^m(\tilde\theta', z')\|\leq L\|(\tilde\theta, z) - (\tilde\theta', z')\|.
\end{align*}
\end{trivlist}

We now consider the continuous-time system of SDEs that correspond to the gradient flow:
\begin{align}
    \dd \tilde{\theta}_t &= \nabla_{\tilde\theta} \int\frac{1}{M N} \sum_{m=1}^M \sum_{n=1}^N \tilde{\ell}^m(\tilde\theta_t, z) q^{m,n}(\dd z) \, \dd t\label{eq:app-cont-theta-ode} \\
    \dd Z_t^{m,n} &= \nabla_{z} \tilde{\ell}^m(\tilde\theta_t, Z_t^{m,n}) \, \dd t + \sqrt{2} \, \dd W_t^{m,n}, \quad (m,n)\in[M]\times [N]\label{eq:app-cont-z-sde}
\end{align}
where $q_t^{m,n}:=\mathrm{Law}(Z_t^{m,n})$ is the law of the particles at time $t$. Since the $Z_t^{m,n}$ are i.i.d., we note that the first equation is equivalent to $\dd \tilde\theta = \nabla_{\tilde\theta} \int {M}^{-1}\sum_{m=1}^M \tilde{\ell}^m(\tilde\theta_t, z) q_t^m(\dd z) \, \dd t$, where $q_t^m$ is the law of $Z_t^{m,n}$ for any $n$.
\begin{lemma}\label{lem:cont-time-convg}
    Under Assumptions~\Cref{assump:model_reg}, \Cref{assump:strong_log_concave}, and \Cref{assump:lipschitz_grad}, the system of SDEs in~\eqref{eq:app-cont-theta-ode}-\eqref{eq:app-cont-z-sde} admits a unique strong solution $(\tilde\theta_t, Z_t^{1:M,1:N})_{t \geq 0}$ and the solution satisfies
    \begin{equation*}
        \mathsf{d}_{\mathcal{M}^{1:M}}((\tilde\theta_t, Q_t^{1:M,N}), (\tilde\theta_\star, Q_\star^{1:M,N})) \leq C e^{-\lambda t}, \quad \forall t>0,
    \end{equation*}
    where $C>0$ is a constant independent of $N$ and $t$, and we defined the empirical distributions to $q_t^m$ and $q_\star^m$ as $Q_t^{m,N}:=N^{-1}\sum_{n=1}^N \delta_{Z_t^{m,n}}$ and $Q_\star^{m,N}:=N^{-1}\sum_{n=1}^N \delta_{Z_\star^{m,n}}$ respectively, where $Z_\star^{m,n}$ are i.i.d. samples from $\tilde\pi_{\tilde\theta_\star}$ (cf. Remark~\ref{remark:star_notation}).
\end{lemma}
\begin{proof}
    The existence and uniqueness of a strong solution follows from Proposition 8 in  \citet{caprioerrorboundsparticle2024} by replacing the system of SDEs therein with ours. To prove the inequality, we note that by a coupling argument for each pair $(Z_t^{m,n}, Z_\star^{m,n})$, we have:
    \begin{align*}
        \mathsf{d}_{\mathcal{M}^{1:M}}((\tilde\theta_t, Q_t^{1:M,N}), (\tilde\theta_\star, Q_\star^{1:M,N}))^2 &\leq \|\tilde\theta_t-\tilde\theta_\star\|^2 + \frac{1}{MN} \sum_{m=1}^{M}\sum_{n=1}^{N} \mathbb{E}[\|Z_t^{m,n}-Z_\star^{m,n}\|^2] \\
        &=\|\tilde\theta_t-\tilde\theta_\star\|^2 + \frac1M \sum_{m=1}^{M}\mathbb{E}[\|Z_t^{m,1}-Z_\star^{m,1}\|^2] \\
        &= \mathsf{d}_{\mathcal{M}^{1:M}}((\tilde\theta_t, q_t^{1:M}), (\tilde\theta_\star, q_\star^{1:M}))^2 \leq \frac{2e^{-2\lambda t}}{\lambda} [\tilde{F}(\tilde\theta_0, q_0) - \tilde{F}_\star],
    \end{align*}
    where we used~\Cref{thm:exp_convg_appendix} in the last inequality.
\end{proof}

\begin{lemma}[Spatial Discretization Error]\label{lem:spatial_discretization}
    Suppose Assumptions~\Cref{assump:strong_log_concave} and~\Cref{assump:lipschitz_grad} hold. Then the following system of SDEs:
    \begin{align}
        \dd \tilde\theta_t^N &= \frac{1}{MN} \sum_{m=1}^{M} \sum_{n=1}^N \nabla_{\tilde\theta} \tilde{\ell}^m(\tilde\theta_t^N, \bar{Z}_t^{m,n}) \, \dd t \label{eq_app:spatial1} \\
        \dd \bar{Z}_t^{m,n} &= \nabla_{z} \tilde{\ell}^m(\tilde\theta_t^N, \bar{Z}_t^{m,n}) \, \dd t + \sqrt{2} \, \dd W_t^{m,n}, \quad (m,n)\in[M]\times [N], \label{eq_app:spatial3}
    \end{align}
    has a strong solution $(\tilde\theta_t^N, \bar{Z}_t^{1:M, 1:N})_{t \geq 0}$. Furthermore, there exists a constant $C(N)>0$ of order $\mathcal{O}(N^{-1/2})$ independent of $t$ and $M$, such that:
    \begin{equation*}
        \mathsf{d}_{\mathcal{M}^{1:M}}((\tilde\theta_t^N, \bar{Q}_t^{1:M, N}), (\tilde\theta_t, Q_t^{1:M,N})) \leq C(N), \quad \forall t>0,
    \end{equation*}
    where $\bar{Q}_t^{m, N}:=N^{-1}\sum_{n=1}^N \delta_{\bar{Z}_t^{m,n}}$ for each $m\in [M]$ and $Q_t^{m,N}$ is defined as in Lemma~\ref{lem:cont-time-convg}.
\end{lemma}
\begin{proof}
    The proof is a modification of~\citet[Lemma 13]{caprioerrorboundsparticle2024} by replacing the SDEs therein with \eqref{eq_app:spatial1}-\eqref{eq_app:spatial3}. More specifically, we define the quantity $\xi^N_t$ as:
    \begin{equation}
        \xi_t^N := \|\tilde\theta_t^N - \tilde\theta_t\|^2 + \frac1M \sum_{m=1}^M \frac1N \sum_{n=1}^N \|Z_t^{m,n} - \bar{Z}_t^{m,n}\|^2,
    \end{equation}
    and we will show that $\mathbb{E}[\xi_t^N] \leq C(N)$ for some constant $C(N)$ which upper bounds the spatial discretization error. Using Itô's formula:
    \begin{align}
        \dd \|\tilde\theta_t^N - \tilde\theta_t\|^2 &= 2 \left\langle \tilde\theta_t^N - \tilde\theta_t, \frac{1}{MN} \sum_{m=1}^{M} \sum_{n=1}^N \nabla_{\tilde\theta} \tilde{\ell}^m(\tilde\theta_t^N, \bar{Z}_t^{m,n}) - \nabla_{\tilde\theta} \int \frac1M \sum_{m=1}^M \tilde{\ell}^m(\tilde\theta_t, z) q_t^m(\dd z) \right\rangle \dd t, \label{eq_app:spatial_error_theta}\\
        \dd \|\bar{Z}_t^{m,n} - {Z}_t^{m,n}\|^2 &= 2 \left\langle  
            \bar{Z}_t^{m,n} - {Z}_t^{m,n}, \nabla_z \tilde{\ell}^m(\tilde\theta_t^N, \bar{Z}_t^{m,n}) - \nabla_z \tilde{\ell}^m(\tilde\theta_t, Z_t^{m,n})
        \right\rangle \dd t.\label{eq_app:spatial_error_z}
    \end{align}
    Adding and subtracting $(MN)^{-1}\sum_m\sum_n\nabla_{\tilde\theta} \tilde{\ell}^m(\tilde\theta_t, Z_t^{m,n})$ from~\eqref{eq_app:spatial_error_theta} and summing with~\eqref{eq_app:spatial_error_z} scaled by $(MN)^{-1}$ yields:
    \begin{align*}
        \dd \xi_t^N &= \frac{2}{MN} \sum_{m=1}^M \sum_{n=1}^N 2\bigl[\left\langle \tilde\theta_t^N - \tilde\theta_t, \nabla_{\tilde\theta} \tilde{\ell}^m(\tilde\theta_t^N, \bar{Z}_t^{m,n}) - \nabla_{\tilde\theta} \tilde{\ell}^m(\tilde\theta_t, Z_t^{m,n}) 
        \right\rangle \dd t \\
        &+\left\langle \bar{Z}_t^{m,n} - {Z}_t^{m,n}, \nabla_z \tilde{\ell}^m(\tilde\theta_t^N, \bar{Z}_t^{m,n}) - \nabla_z \tilde{\ell}^m(\tilde\theta_t, Z_t^{m,n})\right\rangle
        \bigr] \dd t + 2 G_t^N \dd t,
    \end{align*}
    where we defined the term $G_t^N$ as:
    \begin{align*}
        G_t^N &:= \frac1M \sum_{m=1}^M G_t^{m,N}, \quad \text{where}\  G_t^{m,N} := \left\langle \tilde\theta_t^N - \tilde\theta_t, \frac1N \sum_{n=1}^{N} \nabla_{\tilde\theta} \tilde{\ell}^m(\tilde\theta_t, Z_t^{m,n}) - \nabla_{\tilde\theta} \int \tilde{\ell}^m(\tilde\theta_t, z) q_t^m(\dd z) \right\rangle.
    \end{align*}
    Using the strong log-concavity~\Cref{assump:strong_log_concave}, 
    and taking the expectation yields
    \begin{align*}
        \dd \mathbb{E}[\xi_t^N] &\leq -2\lambda \mathbb{E}[\xi_t^N] \dd t + 2 \mathbb{E}[G_t^N] \dd t.
    \end{align*}
    Following the same argument as in~\citet[Lemma 13]{caprioerrorboundsparticle2024}, together with the analogous second-moment bound from~\citet[Proposition 26]{caprioerrorboundsparticle2024} and the $L$-Lipschitz gradient assumption~\Cref{assump:lipschitz_grad}, we can upper bound $|\mathbb{E}[G_t^{m,N}]|$ for each $m\in [M]$:
    \begin{equation}
        \left|\mathbb{E}[G_t^{m,N}]\right|\leq L \sqrt{\frac{2\mathbb{E}\left[\|\tilde\theta_t^N - \tilde\theta_t\|^2\right]}{N} \left(\|\tilde\theta_0\|^2 + \mathbb{E}[\|Z_0\|^2] + \frac{d_z}{\lambda}\right)}\leq L \sqrt{\frac{2\mathbb{E}[\xi_t^N]}{N} \left(\|\tilde\theta_0\|^2 + \mathbb{E}[\|Z_0\|^2] + \frac{d_z}{\lambda}\right)}.
    \end{equation}
    Now we have the following differential inequality for $\mathbb{E}[\xi_t^N]$:
    \begin{equation}
        \frac{\dd}{\dd t} \mathbb{E}[\xi_t^N]^{1/2} \leq -\lambda \mathbb{E}[\xi_t^N]^{1/2} + L \sqrt{\frac{2}{N} \left(\|\tilde\theta_0\|^2 + \mathbb{E}[\|Z_0\|^2] + \frac{d_z}{\lambda}\right)}.
    \end{equation}
    Applying Gr\"onwall's lemma yields the result:
    \begin{equation}
        \mathbb{E}[\xi_t^N]^{1/2} \leq e^{-\lambda t} \mathbb{E}[\xi_0^N]^{1/2} + \frac{(1-e^{-\lambda t})L}{\lambda} \sqrt{\frac{2}{N} \left(\|\tilde\theta_0\|^2 + \mathbb{E}[\|Z_0\|^2] + \frac{d_z}{\lambda}\right)},
    \end{equation}
    where the first term is zero due to construction.
\end{proof}

\begin{lemma}[Temporal Discretization Error]\label{lem:temporal_discretization}
    Suppose Assumptions~\Cref{assump:strong_log_concave} and~\Cref{assump:lipschitz_grad} hold. Then for the following Euler-Maruyama scheme for the SDEs in \eqref{eq_app:spatial1}-\eqref{eq_app:spatial3}:
    \begin{align}
        \tilde\theta_{k+1}^{N,h} &= \tilde\theta_k^{N,h} + \frac{h}{MN} \sum_{m=1}^{M}\sum_{n=1}^N \nabla_{\tilde\theta} \tilde{\ell}(\tilde\theta_k^{N,h}, \bar{Z}_k^{m,n,N,h}) \label{eq_app:temp1} \\
        \bar{Z}_{k+1}^{m,n,N,h} &= \bar{Z}_k^{m,n,N,h} + h \nabla_{z} \tilde{\ell}(\tilde\theta_k^{N,h}, \bar{Z}_k^{m,n,N,h}) + \sqrt{2h} W_k^n, \quad (m,n) \in [M] \times [N], \label{eq_app:temp3}
    \end{align}
    we have for $h \leq {1}/({\lambda+L})$ with $\lambda$ being the strong concavity constant in Assumption~\Cref{assump:strong_log_concave} and $L$ being the Lipschitz constant in \Cref{assump:lipschitz_grad}, there exists a constant $C(h)>0$ independent of $k$ and $N$ of order $\mathcal{O}(h^{1/2})$,
    such that:
    \begin{equation*}
        \mathsf{d}_{\mathcal{M}^{1:M}}((\tilde\theta_k^{N,h}, \bar{Q}_k^{1:M,N,h}), (\tilde\theta_{kh}^N, \bar{Q}_{kh}^{1:M,N})) \leq \sqrt{h}C(h), \quad \forall k \in \mathbb{N},
    \end{equation*}
\end{lemma}
\begin{proof}
    The proof is a direct adaptation of~\citet[Lemma 14]{caprioerrorboundsparticle2024} by replacing the discretization therein with~\cref{eq_app:temp1,eq_app:temp3} and using the same argument for the $M\times N$ equations for the particles and replacing the loss with $M^{-1}\sum_{m=1}^M \tilde{\ell}^m$ for the parameters.
\end{proof}

\begin{theorem}\label{thrm_app:error-bound}
Suppose that the premise of Theorem~\ref{thm:exp_convg_informal} and \Cref{assump:lipschitz_grad} hold, and that $\mathcal{R}$ has Lipschitz gradients. For all sufficiently small $h>0$, there exists an $\mathcal{O}(h^{1/2}+N^{-1/2})$ constant $C_{h,N}$ independent of $T$, a $\rho \in (0, 1)$, and a $C>0$ independent of $(h,N,T)$, such that
$$\mathbb{E}\left[\vert\vert\tilde\theta_T-\tilde\theta_\star \vert\vert^2\right]^{1/2} \leq C_{h,N}+C\rho^T\quad\forall T\in \mathbb{N},$$
where $\tilde\theta_\star$ denotes $\tilde{\ell}$'s unique maximizer.
\end{theorem}
\begin{proof}
    The proof follows from a combination of Lemma~\ref{lem:cont-time-convg}, Lemma~\ref{lem:spatial_discretization}, and Lemma~\ref{lem:temporal_discretization} by an application of the triangle inequality and the definition of the metric $\mathsf{d}_{\mathcal{M}^{1:M}}$.
\end{proof}

\section{Additional Results}
\label{app:additional_results}
\subsection{Interpolating the Latent Space}
As a deep \gls*{LVM}, \gls*{IPLD} is able to learn a smooth latent space, enabling semantically meaningful interpolations. In \Cref{fig:train_intp}, we take two indices $m_1, m_2 \in [M]$ from the training set indices and extract the particles $z_{0}^{m_1, 0}, z_{0}^{m_2,0}$ from the particle cloud. We compute the linear interpolation via $\texttt{lerp}(x,y,s)=s x + (1-s) y$ between those particles for $s\in[0,1]$ and decode back into the pixel space with the decoder $g_\phi$.
\begin{figure}[htp]
    \centering
    \includegraphics[width=0.7\linewidth]{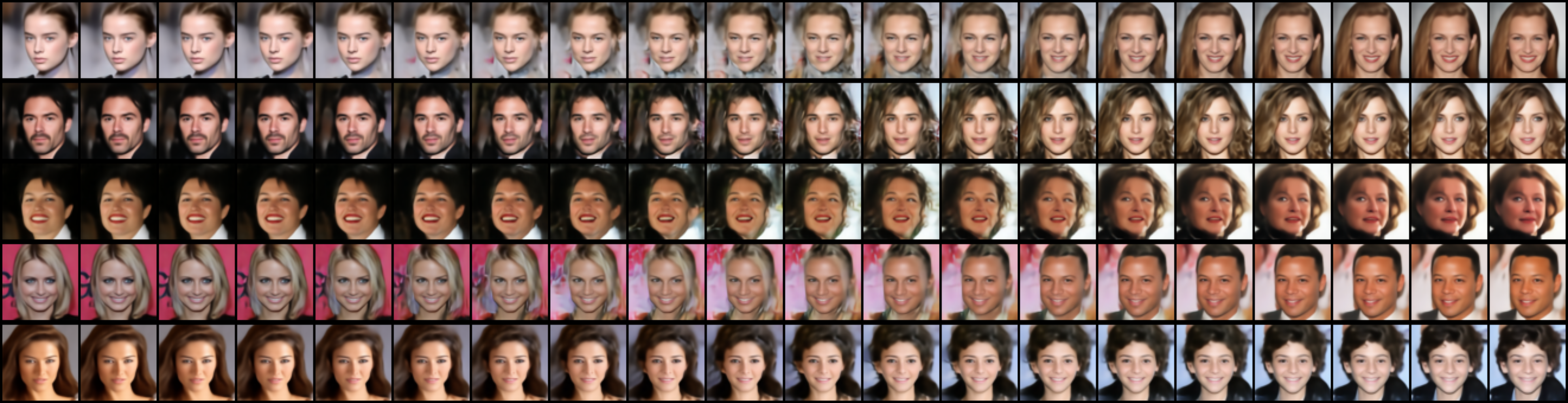}
    \caption{Linear interpolation between the particles learned on the CelebA64 training set.}
    \label{fig:train_intp}
\end{figure}
\subsection{Evolution of the Particle Cloud}
We also visualize how the particle cloud $z_{0}^{1:M, 1:N}$ evolves through the gradient flow. To achieve this, we save the particle cloud at training epochs $\{0,5,10,15,20,25,30,50,75,90\}$ and pass them through the decoder $g_\phi$ (\Cref{fig:evo_cifar,fig:evo_celeba}). We note that there is an intriguing perceptual similarity between the particle evolution and the posterior mean $\mathbb{E}[ x_0 | x_t]$ of a diffusion model. For comparison, we show the evolution of the posterior mean predicted in \Cref{fig:evo_ddim_cifar,fig:evo_ddim_celeba}. The posterior mean is computed by a pretrained score network taken from \citet{song2022denoisingdiffusionimplicitmodels} in the pixel space via Tweedie's formula \citep{efron2011tweedie}: $\mathbb{E}[ x_0| x_t] \approx  (\sqrt{\bar{\alpha}_t})^{-1} [  x_t + \sigma_t^2 \nabla_{ x_t} \vs_\theta( x_t, t)]$, where $\sigma_t^2$ is the variance of the reverse process and $\vs_\theta( x_t, t)$ is the score estimating network.
We remark that there have been several recent works attempting to delineate the connection between diffusion and gradient flow \citep{yiMonoFlowRethinkingDivergence2023,huangGANsGradientFlows2023,franceschiUnifyingGANsScorebased2023}. Our method can be thought as an attempt of learning the gradient flow via a diffusion model.
\begin{figure}[htbp] 
    \centering
    \begin{subfigure}{0.35\linewidth}
         \centering
         \includegraphics[width=\linewidth]{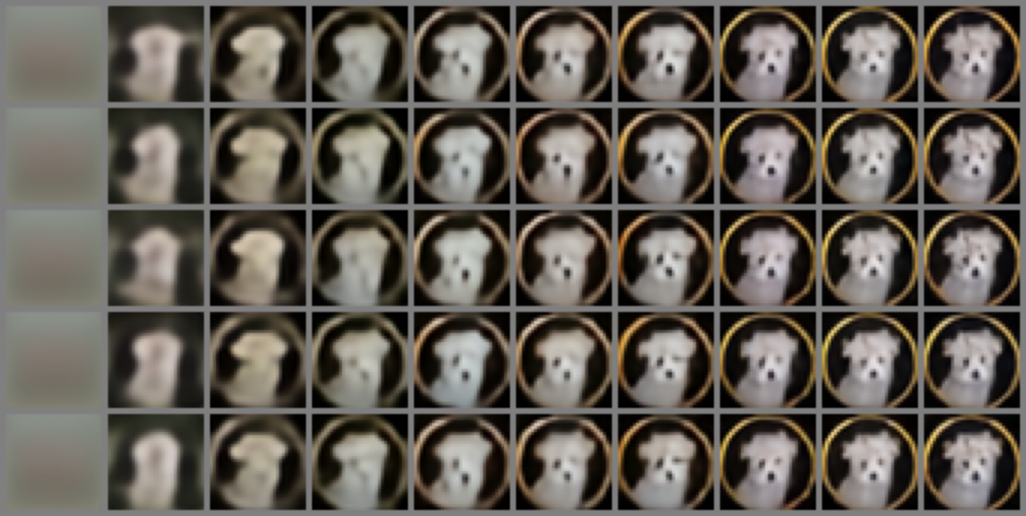}
         \caption{CIFAR-10}
         \label{fig:evo_cifar}
    \end{subfigure}
    \hspace{0.2cm}
    \begin{subfigure}{0.35\linewidth}
         \centering
         \includegraphics[width=\linewidth]{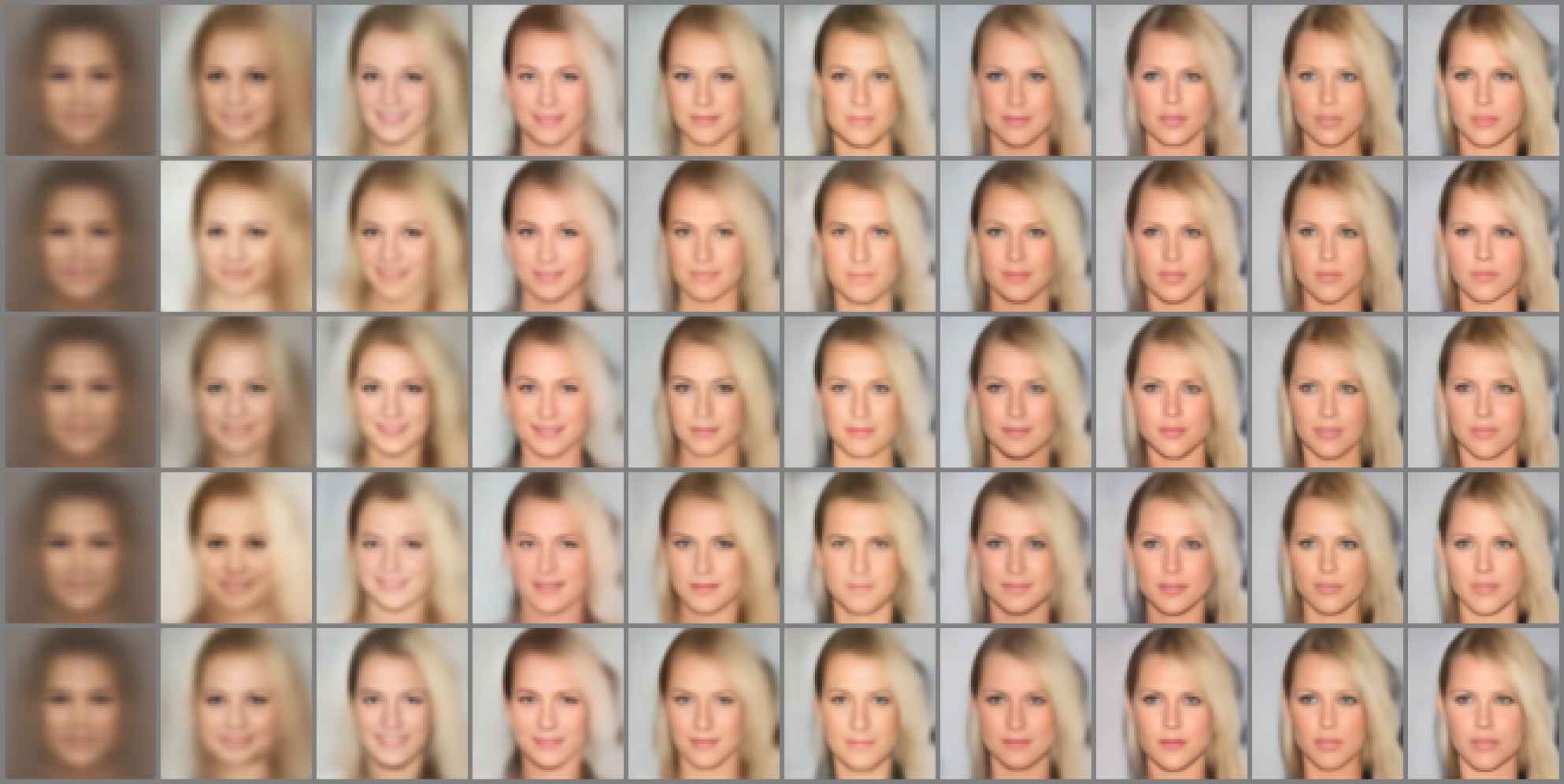}
         \caption{CelebA64}
         \label{fig:evo_celeba}
    \end{subfigure}
    \caption{Evolution of the particle cloud of \gls*{IPLD} trained on the CIFAR-10 and CelebA64 dataset with $5$ particles. Zoom in to view the details better.}
\end{figure}
\begin{figure}[htbp]
    \centering
    \begin{subfigure}{0.35\linewidth}
         \centering
         \includegraphics[width=\linewidth]{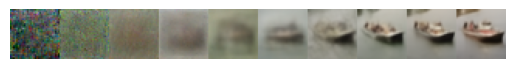}
         \caption{CIFAR-10}
         \label{fig:evo_ddim_cifar}
    \end{subfigure}
    \hspace{0.2cm}
    \begin{subfigure}{0.35\linewidth}
         \centering
         \includegraphics[width=\linewidth]{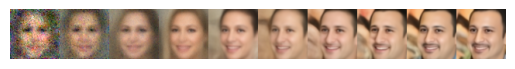}
         \caption{CelebA64}
         \label{fig:evo_ddim_celeba}
    \end{subfigure}
    \caption{Evolution of the predicted $\mathbb{E}[x_0| x_t]$ of \gls*{DDPM} trained on the CIFAR-10 and CelebA64 dataset.}
\end{figure}

\subsection{Ablation on Diffusion Weighting}
As reported in \citet{hodenoisingdiffusionprobabilistic2020,kingmaunderstandingdiffusionobjectives2023}, using the simplified diffusion objective can yield samples with better visual quality. To verify this effect within our framework, we compare two variants of our model: \gls*{IPLD} likelihood using the diffusion loss in \eqref{app_eq:diffusion_likelihood} and \gls*{IPLD} usual using the loss in \eqref{eq_app:simplified_diffusion}. We train all models with the same training hyperparameters specified in \Cref{app:image_exp_details}.
\begin{table}[htp]
    \centering
    \begin{tabular}{lcc}
\hline
& SVHN           & CIFAR-10        \\ \hline
\gls*{IPLD} usual & 17.55     &  51.60     \\
\gls*{IPLD} likelihood & 18.62 & 53.22  \\
\hline
\end{tabular}
    \vspace{0.2cm}
    \caption{Ablation study on diffusion weighting.  The final \gls*{FID} scores are reported.}
    \label{tab:app_ablate_likelihood}
\end{table}
The results in Table \ref{tab:app_ablate_likelihood} show that the simplified objective yields a consistently better FID score, which aligns with prior findings.

\subsection{Ablation on Two Stage Training}
\label{app:2stage_ablate}
\citet{leng2025repaeunlockingvaeendtoend} hypothesized that backpropogating the diffusion loss to the \gls*{VAE} directly make the latent space simpler. They suggested this could inadvertently "hack" the denoising objective, leading the diffusion model to simply predict noise from the \gls*{VAE}'s Gaussian approximate posterior.

To investigate this, we conduct an ablation study comparing end-to-end training and two-stage training. We train both models with $1$ particle using the same training hyperparameters; for two-stage training we detach the gradients on the particles before passing them to the diffusion model, hence preventing backpropogating the diffusion loss. We show in \Cref{tab:app_ablate_2stage} that end-to-end training in \gls*{IPLD} does not suffer from the same issue observed in \citet{leng2025repaeunlockingvaeendtoend}. This aligns with the results in \citet{vahdatscorebasedgenerativemodeling2021}, who also found benefits to end-to-end training. It is important to note, however, that both our experiments and those of \citet{vahdatscorebasedgenerativemodeling2021} were conducted on relatively lower-resolution datasets. Therefore, the conclusions drawn here may not directly extrapolate to the larger-scale experimental settings used in \citet{leng2025repaeunlockingvaeendtoend}.
\begin{table}[htp]
    \centering
    \begin{tabular}{lccc}
\hline
& SVHN           & CIFAR-10      & CelebA64    \\ \hline
\gls*{IPLD} usual & 17.55     &  51.60    & 22.86   \\
\gls*{IPLD} detached & 19.07 & 52.96 &  23.19 \\
\hline
\end{tabular}
    \vspace{0.2cm}
    \caption{Ablation study on two-stage training.  The final \gls*{FID} scores are reported.}
    \label{tab:app_ablate_2stage}
\end{table}

\subsection{Runtime and Memory Comparisons}
\label{app:runtime-compare}
In \Cref{tab:single_gpu_performance}, we benchmark the peak memory usage and walltime for a single forward-backward pass with a batch size of 64 on a single A6000 GPU (48GB) for both \gls*{DiffusionVAE} and \gls*{IPLD}:
\begin{table}[h!]
\centering
\caption{Performance comparison on a single A6000 GPU. Lower values are better.}
\label{tab:single_gpu_performance}
\begin{tabular}{
  l
  S[table-format=2.2]
  S[table-format=1.2]
}
\toprule
\textbf{Method} & {\textbf{Peak Memory (GB) $\downarrow$}} & {\textbf{Walltime (s) $\downarrow$}} \\
\midrule
\gls*{DiffusionVAE} 1-particle  & 5.19  & 0.15 \\
\gls*{DiffusionVAE} 5-particle  & 17.51 & 0.54 \\
\gls*{DiffusionVAE} 10-particle & 33.89 & 1.01 \\
\midrule
\gls*{IPLD} 1-particle          & 4.19  & 0.12 \\
\gls*{IPLD} 5-particle          & 17.26 & 0.49 \\
\gls*{IPLD} 10-particle         & 33.62 & 0.98 \\
\bottomrule
\end{tabular}
\end{table}
Using \texttt{torchrec} \citep{torchrec} and custom sharding strategies enabled by it, we implement a distributed version of \gls*{IPLD} by allocating each accelerator a different subset of the particles (cf. \Cref{sec:pratical-alg}). The results in \Cref{tab:distributed_gpu_performance} confirms the scalability of our algorithm.
\begin{table}[h!]
\centering
\caption{Performance of our custom distributed \gls*{IPLD} implementation on two A6000 GPUs. Lower values are better.}
\label{tab:distributed_gpu_performance}
\begin{tabular}{
  l
  S[table-format=2.2]
  S[table-format=1.2]
}
\toprule
\textbf{Method} & {\textbf{Peak Memory (GB) $\downarrow$}} & {\textbf{Walltime (s) $\downarrow$}} \\
\midrule
\gls*{IPLD} 1-particle  & 2.81  & 0.09 \\
\gls*{IPLD} 5-particle  & 9.62  & 0.29 \\
\gls*{IPLD} 10-particle & 17.95 & 0.56 \\
\bottomrule
\end{tabular}
\end{table}

\subsection{Additional Uncurated Samples}
\label{app:additional_samples}
We present additional uncurated samples of \gls*{IPLD} trained with 10 particles on CIFAR-10, SVHN, and CelebA64 dataset produced by the DDIM sampler \citep{song2022denoisingdiffusionimplicitmodels} with 100 NFEs. We remark that only the CelebA64 samples in the main text have been curated for better visualization.

\begin{figure}[htbp]
    \centering
    \includegraphics[width=0.5\linewidth]{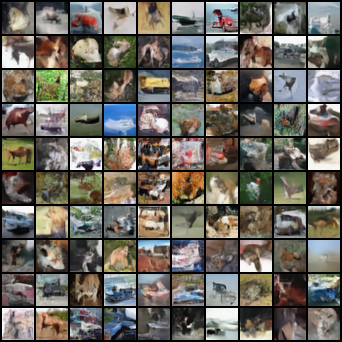}
    \caption{Uncurated samples on CIFAR-10.}
\end{figure}

\begin{figure}[htbp]
    \centering
    \includegraphics[width=0.5\linewidth]{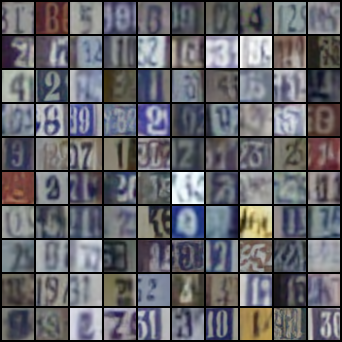}
    \caption{Uncurated samples on SVHN.}
\end{figure}
\begin{figure}[htbp]
    \centering
    \includegraphics[width=0.5\linewidth]{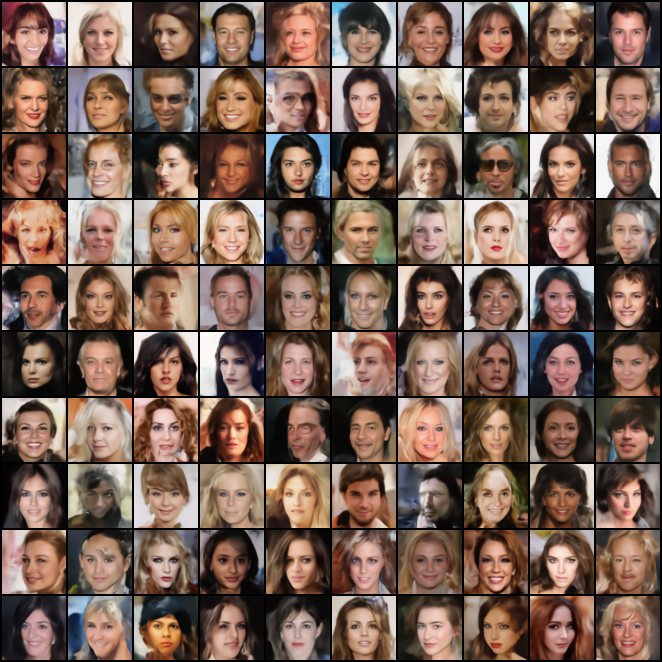}
    \caption{Uncurated samples on CelebA64.}
\end{figure}

\end{document}